\documentclass[letterpaper]{article} 
\usepackage[]{aaai25}  
\usepackage{times}  
\usepackage{helvet}  
\usepackage{courier}  
\usepackage[hyphens]{url}  
\usepackage{graphicx} 
\usepackage{placeins}
\usepackage{natbib}  
\usepackage{caption} 
\usepackage{caption}
\usepackage{subcaption}
\usepackage{pgffor}
\usepackage{booktabs}
\usepackage{xcolor}

\usepackage{bm} 
\usepackage{multirow} 
\usepackage{tabularx}
\usepackage{relsize}
\usepackage{algorithm}
\usepackage{algorithmic}
\usepackage{amsfonts}
\usepackage{blkarray}
\usepackage{amsmath, amssymb, amsthm}
\usepackage{wasysym}
\usepackage{array}

\usepackage{newfloat}
\usepackage{listings}
\DeclareCaptionStyle{ruled}{labelfont=normalfont,labelsep=colon,strut=off} 
\floatstyle{ruled}
\newfloat{listing}{tb}{lst}{}
\floatname{listing}{Listing}
\title{Multi-Subspace Matrix Recovery from Permuted Data}
\author{
    Liangqi Xie,
    Jicong Fan\thanks{Corresponding author.}
}
\affiliations{
    School of Data Science, The Chinese University of Hong Kong, Shenzhen\\
    222041028@link.cuhk.edu.cn, fanjicong@link.cuhk.edu.cn
}

\usepackage{bibentry}

\begin{document}

\maketitle

\newtheorem{theorem}{Theorem}
\newtheorem{corollary}{Corollary}
\newtheorem{definition}{Definition}
\newtheorem{lemma}{Lemma}
\newtheorem{proposition}{Proposition}
\newtheorem{assumption}{Assumption}

\begin{abstract}
This paper aims to recover a multi-subspace matrix from permuted data: given a matrix, in which the columns are drawn from a union of low-dimensional subspaces and some columns are corrupted by permutations on their entries, recover the original matrix. The task has numerous practical applications such as data cleaning, integration, and de-anonymization, but it remains challenging and cannot be well addressed by existing techniques such as robust principal component analysis because of the presence of multiple subspaces and the permutations on the elements of vectors. To solve the challenge, we develop a novel four-stage algorithm pipeline including outlier identification, subspace reconstruction, outlier classification, and unsupervised sensing for permuted vector recovery. Particularly, we provide theoretical guarantees for the outlier classification step, ensuring reliable multi-subspace matrix recovery. Our pipeline is compared with state-of-the-art competitors on multiple benchmarks and shows superior performance.
\end{abstract}

%

\section{Introduction}
\subsection{Background and Motivation}

Permutation is a critical form of data corruption in many applications such as computer vision, data integration, and privacy protection, and hence therefore requires significant attention. This section highlights two key applications: record linkage and de-anonymization, related to data integration and privacy protection, respectively.
In record linkage, the goal is to integrate data from different sources for analysis \cite{record_linkage, re-identification}. Columns of a data matrix, gathered independently, may not correspond to the same entity in each row. Thus, reordering or recovering these columns is essential for accurate analysis.
In de-anonymization, data providers anonymize information by shuffling the columns of the ground-truth data matrix before release. Recovering the original data becomes a reverse process of data protection \cite{de-anonymization}. This has practical implications, especially in healthcare and finance, where data integrity is crucial.

Typically, the ground-truth data matrix has a lower rank than the permuted version. However, our study addresses a more general scenario where the ground-truth matrix can be full-rank \cite{FAN2018378,fan2019online,fan2020polynomial}, leading us to explore Permuted Matrix Recovery with Multi-Subspace Data.

\subsection{Related Work}

While several studies address label-entity mismatches, only two focus on matrix recovery. The first, \cite{upca}, combines robust PCA \cite{rpca} with unlabeled sensing, estimating the ground-truth subspace first and using it to recover permuted data. The second, \cite{MRUC}, estimates the permutation matrix within a Birkhoff polytope, minimizing an objective inspired by nuclear norm minimization. This reformulates the problem into continuous optimization within the polytope, solved via proximal gradient methods and the Sinkhorn algorithm \cite{cuturi2013sinkhorn}.

Each method has strengths and limitations. The method proposed by \cite{upca} works well for sparsely permuted data or low-dimensional subspaces but struggles with missing data, whereas the method proposed by \cite{MRUC} handles such data but assumes permuted clusters of outliers, a condition that may not always hold. Additionally, the method of \cite{MRUC} is sensitive to initial conditions and assumes a low-rank subspace, with a time complexity of \( \mathcal{O}(n^2) \), limiting its scalability in high-dimensional scenarios.
Importantly, both methods focus on a single ground-truth subspace, overlooking the more complex scenario involving multi-subspace data (widely existing in many areas such as computer vision and signal processing), which we aim to address. While the robust kernel PCA proposed by \cite{fan2019exactly} removes sparse noise from high-rank matrices, it does not effectively handle permuted data.

\subsection{Contributions of This Work}

We generalize the traditional single-subspace permuted matrix recovery problem into multi-subspace scenarios and propose a four-step pipeline for recovering a multi-subspace matrix corrupted by permutation. Within the pipeline, we introduce an efficient method for the outlier classification step and provide theoretical guarantees to support the algorithm and practical applications.


\section{Methodology}
\subsection{Problem Formulation}
Let $\mathbf{G}\in\mathbb{R}^{M\times N}$ be a clean data matrix of which the columns are randomly drawn from a union of $L$ $r$-dimensional subspaces $\{\mathcal{S}_{k}\}_{k=1}^{L}$, where $1\leq r<M$. Suppose $N_\mathbf{Y}$ columns of $\mathbf{G}$, forming a matrix $\mathbf{Y}\in\mathbb{R}^{M \times N_\mathbf{Y}}$, are corrupted by permutations, that is, for $i\in[N_{\mathbf{Y}}]$, 
\begin{equation}
 \tilde{\mathbf{y}}_i=\mathbf{P}_i\mathbf{y}_i   
\end{equation}
where $\mathbf{P}_i\in\mathcal{P}_M$ is a partial permutation matrix. An illustrative example for $\mathbf{P}_i$ when $M=3$ is $[1 ~0 ~0; 0 ~0 ~1;0 ~1 ~0]$. The rest $N_\mathbf{X}$ columns of $\mathbf{G}$, forming a matrix $\mathbf{X}\in\mathbb{R}^{M \times N_\mathbf{X}}$, remain unchanged, where $N_\mathbf{X} + N_\mathbf{Y}=N$. Then the corrupted data matrix is denoted as $\widetilde{\mathbf{G}}$, consisting the columns of $\mathbf{X}$ and $\widetilde{\mathbf{Y}}$, where $\widetilde{\mathbf{Y}}=[\tilde{\mathbf{y}}_1,\tilde{\mathbf{y}}_2,\ldots,\tilde{\mathbf{y}}_{N_{\mathbf{Y}}}]$.
We call the columns of $\mathbf{\widetilde{Y}}$ \textit{outliers} for convenience. 
Our goal is to recover $\mathbf{G}$ from $\mathbf{\widetilde{G}}$. This task is more challenging if $L$, $r$, or $\frac{N_{\mathbf{Y}}}{N}$ is larger. It has numerous real applications, e.g.:
\begin{itemize}

    \item \textbf{Data Cleaning:} In many real scenarios such as health care \cite{talebi2020ehr}, the samples in a dataset may be drawn from a union of subspaces (corresponding to different groups or clusters) and the attribute names of many samples may be missing or incorrect due to recording mistakes or technical errors. It is important to identify these samples and recover the true orders of the attributes, such that the performance of downstream tasks is reliable.
    

    \item \textbf{Multi Dataset De-anonymization:} Data permutation methods are pivotal in data privacy and anonymization, particularly with multisubspace data \cite{byun2006secure} \cite{ji2016graph}. De-anonymization challenges escalate with the presence of multiple subspaces or latent structures within datasets, which can be exploited for re-identification. This risk intensifies in multi-subspace contexts due to potential overlapping information. Consequently, adapting de-anonymization techniques to address multi-subspace scenarios is essential to mitigate privacy breaches effectively.

\end{itemize}

\subsection{Proposed Method}
Using a single technique such as a denoising algorithm (e.g. robust PCA) to directly recover $\mathbf{G}$ from $\widetilde{\mathbf{G}}$ is often infeasible because several important and commonly-used assumptions such as I.I.D. and low-rankness do not hold in the problem. For example, when there is no overlap between the subspaces, $\mathbf{G}$ is full-rank if $Lr\geq M$. Therefore, we propose a pipeline consisting of four stages to address the challenge. The four stages are outlier detection, subspace clustering and estimation, outlier classification, and matrix recovery, respectively. Our method, termed \textbf{P}ermuted \textbf{M}ulti-\textbf{S}ubspace \textbf{D}ata \textbf{R}ecovery (PMSDR), is summarized in Algorithm \ref{alg:algorithm_4_stage}. In the following context, we elaborate on the four stages.


\subsubsection{Step 1: Outlier Detection from $\widetilde{\mathbf{G}}$}
Outlier detection \cite{hodge2004survey} is a critical step in numerous data analysis tasks. Over the past few decades, a variety of methods have been developed to address this problem, ranging from traditional statistical approaches to more advanced machine learning techniques \cite{boukerche2020outlier}. While these methods can be effective in certain scenarios, they may face challenges when dealing with high-dimensional or complex data. One may consider robust PCA \cite{rpca,fan2019factor} but it requires the low-rank assumption, which does not hold in our problem.


To identify the outliers in $\widetilde{\mathbf{G}}$, we propose to use a method called Provable Self-Representation Matrix (PSRM) given by \cite{you2017provable}. We construct a self-representation matrix $\mathbf{R} = \left(r_{ij}\right) \in \mathbb{R}^{n\times n}$ by solving the following elastic net problem 
\begin{equation*}
    \min_{\operatornamewithlimits{\mathbf{r}_j}} \lambda \|\mathbf{r}_j\|_1 + \tfrac{1-\lambda}{2} \|\mathbf{r}_j\|_2^2 + \tfrac{\gamma}{2} \|\tilde{\mathbf{g}}_j-\widetilde{\mathbf{G}}\mathbf{r}_j\|_2^2 \quad \text{s.t. } r_{jj} = 0,
\end{equation*}
and then build the transition matrix $\mathbf{P} = \left(p_{ij}\right)  \in \mathbb{R}^{M\times M}$ by
\begin{equation}
    p_{ij} = \left|r_{ji}\right|/\|\mathbf{r}_i\|_1 \quad \text{for all} \{i, j\} \subseteq [M],
\end{equation}
thus forming a stochastic process. By initializing with an initial discrete union distribution, the probability mass will, over successive steps, provably concentrate on inliers under certain assumptions. This probabilistic behavior allows us to identify and separate outliers from inliers. We denote the estimated inliers and outliers in $\widetilde{\mathbf{G}}$ as $\widehat{\mathbf{X}}$ and $\widehat{\widetilde{\mathbf{Y}}}$, respectively.

\begin{algorithm}[tb]
    \caption{Permuted Multi-Subspace Data Recovery Pipeline (\textbf{PMSDR})}
    \label{alg:algorithm_4_stage}
    \textbf{Input}: 
    \begin{itemize}
        \item Observed matrix $\mathbf{\widetilde{G}}$ consisting the columns of  $\mathbf{X}$ and $ \mathbf{\widetilde{Y}}$
        \item Recovery rank $r$
        \item Subspace number $L$
    \end{itemize}
    \textbf{Output}: 
    \begin{itemize}
        \item Recovered data matrix $\widehat{\mathbf{G}}$
    \end{itemize}
    \begin{algorithmic}[1] 
        \STATE \textbf{Step 1: Outlier Detection.} Seperate columns of $\mathbf{\widetilde{G}}$ into $\widehat{\mathbf{X}}$ and $\widehat{\mathbf{\widetilde{Y}}}$ using the $l_1$-norm of sparse self-representation coefficients \cite{you2017provable}. \label{Step 1: Outlier Detection.}
        \STATE \textbf{Step 2: Subspace Reconstruction.} 
        Cluster columns of $\widehat{\mathbf{X}}$ into $L$ groups to obtain matrices $\{\widehat{\mathbf{X}}_1, \ldots, \widehat{\mathbf{X}}_L\}$ and estimate an $r$-dim basis $\widehat{\mathbf{U}}_k$ for each $\widehat{\mathbf{X}}_k$ using SVD. 

        \STATE \textbf{Step 3: Outlier Classification.} Associate each outlier $\mathbf{\tilde{y}}_j$ with its respective subspace $\mathcal{S}_{k_j}$ using Algorithm \ref{alg:solve_partial_lsr}. \label{step4:outlier-classification}
        \STATE \textbf{Step 4: Matrix Recovery.} Recover each $\mathbf{\tilde{y}}_j$ w.r.t. its corresponding subspace $\mathcal{S}_{k_j}$ using unsupervised sensing techniques such as UPCA (our default approach) \cite{upca}.
        \label{step4:matrix-recvery}

        \STATE \textbf{Return}: $\widehat{\mathbf{Y}} = [\mathbf{\hat{y}}_1, \ldots, \mathbf{\hat{y}}_{N_{\widehat{\mathbf{Y}}}}]$ and $\widehat{\mathbf{G}} = [\widehat{\mathbf{X}}, \widehat{\mathbf{Y}} ]$
    \end{algorithmic}
\end{algorithm}

\begin{algorithm*}[t]
\caption{Outlier Classification}
\label{alg:solve_partial_lsr}
\textbf{Input}: 
Estimated bases for all subspaces $\widehat{\mathbf{U}}_1, \ldots, \widehat{\mathbf{U}}_L \in \mathbb{R}^{M \times r}$; One outlier sample $\mathbf{\tilde{y}} \in \mathbb{R}^M$\\
\textbf{Initialized Parameters}: Retain ratio $\gamma$; Maximum iterations $\textit{max\_iter}$\\
\textbf{Output}: 
Corresponding subspace label $t \in [L]$

\begin{algorithmic}[1] 
    \STATE Initialize elimination numbers $m = [m_1, m_2, \ldots, m_{\textit{iter}}]$, where each $m_i \in \mathbb{Z}^{+}$ is in descending order, $\textit{iter} \leq \textit{max\_iter}$, and $\sum_{i = 1}^{\textit{iter}} m_i = \lfloor (n-r)*(1-\gamma) \rfloor$.
    \FOR{$k = 1$ to $L$}
        \STATE Initialize $\bm{\nu}^{(0)} = \mathbf{\tilde{y}}$ and $\mathbf{B}^{(0)} = \widehat{\mathbf{U}}_k$.
        \FOR{$i = 1$ to \textit{iter}}
            \STATE $[\hat{j}_1, \hat{j}_2, \ldots, \hat{j}_{m_i}] = \operatorname*{argmax}_{j_1, j_2, \ldots, j_{m_i}} \left| \bm{\nu}^{(i-1)} - \mathbf{B}^{(i-1)} {\mathbf{B}^{(i-1)}}^{\dagger} \bm{\nu}^{(i-1)} \right|$
            \STATE Remove the $[\hat{j}_1, \hat{j}_2, \ldots, \hat{j}_{m_i}]$-th entries from $\bm{\nu}^{(i-1)}$ to get $\bm{\nu}^{(i)}$.
            \STATE Remove the $[\hat{j}_1, \hat{j}_2, \ldots, \hat{j}_{m_i}]$-th rows from $\mathbf{B}^{(i-1)}$ to get $\mathbf{B}^{(i)}$ and refined subspace $\mathcal{S}_k^{(i)}$.
        \ENDFOR
        \STATE Calculate the subspace distance $d_k = 1 - \cos\left(\bm{\nu}^{(\textit{iter})}, \mathcal{S}_k^{(\textit{iter})}\right)$
    \ENDFOR
    \STATE Determine the subspace label $t = \operatorname*{argmin}_k d_k$
\end{algorithmic}
\end{algorithm*}

\subsubsection{Step 2: Subspace Reconstruction}
Given $\widehat{\mathbf{X}}$, we need to cluster its columns into $L$ groups to obtain matrices $\{\widehat{\mathbf{X}}_1, \ldots, \widehat{\mathbf{X}}_L\}$ corresponding to different subspaces. There is a large literature on the issue of subspace clustering in recent decades \cite{SSC_a,liu2012robust,kfsc2021kdd,Cai_2022_CVPR}. Here we simply use SSC \cite{SSC_a} as the subspace clustering method. One can utilize any other method as an alternative when needed. It's also worth noting that like most subspace clustering methods, the more samples there are, the better the performance is, which indicates a future direction to enhance the performance of subspace clustering by other tricky techniques. 


For subspace estimation, we utilize the Singular Value Decomposition (SVD) to compute basis vectors $\widehat{\mathbf{U}}_k$ for each subspace $\mathcal{S}_k$, i.e., $\mathbf{X}_k = \mathbf{U}_k \mathbf{\Sigma}_k \mathbf{V}_k^\top$, where we define $\widehat{\mathbf{U}}_k$ as the first $r$ columns of $\mathbf{U}_k$. These basis vectors are arranged in descending order based on their corresponding singular values. Therefore, the process involves selecting the top \( r \) eigenvectors to form a basis for \( \mathcal{S} \).
Alternatives like DPCP \cite{DPCP} can also be used.

\subsubsection{Step 3: Outlier Classification}
To recover $\mathbf{Y}$ from $\widehat{\widetilde{\mathbf{Y}}}$ using $\{\widehat{\mathbf{U}}_k\}_{k=1}^{L}$,  we need to find the corresponding $\widehat{\mathbf{U}}_k$ or subspace for each column of $\widehat{\widetilde{\mathbf{Y}}}$ first. This is essentially a classification task but the elements of each column in $\widehat{\widetilde{\mathbf{Y}}}$ are partially permuted, which leads to a considerable challenge. 


Inspired by the UPCA method proposed by \cite{upca}, we propose an efficient algorithm to resolve the challenge, shown in Algorithm \ref{alg:solve_partial_lsr}. It iteratively eliminates unimportant entries to identify the true correspondence between an outlier and its subspace. 
Given an outlier 
\begin{equation}
\mathbf{\tilde{y}} = 
\begin{bmatrix}
    \mathbf{y^{(1)}} \\
    \mathbf{\tilde{y}^{(2)}}
\end{bmatrix}
\end{equation}
with $\mathbf{\tilde{y}^{(2)}}$ being fully permuted, and a basis of one subspace 
\begin{equation}
\widehat{\mathbf{U}}_k = 
\begin{bmatrix}
    \mathbf{\widehat{U}}_k^{(1)} \\
    \mathbf{\widehat{U}}_k^{(2)}
\end{bmatrix} 
\quad (k=1,\ldots,L),
\end{equation}
the core idea is to eliminate entries in $\mathbf{\tilde{y}^{(2)}}$ while retaining $\mathbf{y^{(1)}}$ as completely as possible. Then, we compare the cosine distance between the retained vector $\mathbf{y^{(1)}}$ and each subspace $\mathcal{S}_k$ using the formula:
\begin{equation}
d_k = 1 - \cos\left( \mathbf{y^{(1)}}, \mathbf{\widehat{U}}_k^{(1)} \left( \mathbf{\widehat{U}}_k^{(1)} \right)^{\dagger} \mathbf{y^{(1)}} \right),
\end{equation}
and select the subspace with the minimum distance as the estimated subspace class.

Specifically, the algorithm first initializes the total number of steps \textit{iter} and the number of eliminated entries $m_i$ in $i$-th step, which significantly enhances the efficiency. For each subspace $\mathcal{S}_k$, the algorithm initializes the remaining vector $\bm{\nu}^{(0)} = \mathbf{\tilde{y}}$ and the basis matrix $\mathbf{B}^{(0)} = \widehat{\mathbf{U}}_k$. During each iteration, the algorithm performs least squares regression:
\begin{equation}
\bm{\hat{\nu}}^{(i-1)} = \mathbf{B}^{(i-1)} \left( \mathbf{B}^{(i-1)} \right)^{\dagger} \bm{\nu}^{(i-1)}
\end{equation}
and removes the largest $m_i$ entries in the residual
\begin{equation}
\left| \bm{\nu}^{(i-1)} - \bm{\hat{\nu}}^{(i-1)} \right|
\end{equation}
from both $\bm{\nu}$ and $\mathbf{B}$. After completing the iterations, the subspace distance is calculated to determine the subspace label $t$ that minimizes the cosine distance between the remaining vector $\bm{\nu}^{(\text{iter})}$ and the subspace $\mathcal{S}^{(\text{iter})}_k$.

Intuitively, the outlier classification method is easy to understand. The elimination of the largest residuals in each iteration ensures that the remaining data points better represent the underlying subspace structure. By iteratively refining the basis matrix $\mathbf{B}$ and the residual vector $\bm{\nu}-\bm{\hat{\nu}}$, the algorithm effectively isolates the outlier and aligns it with the correct subspace.
From a theoretical perspective, under mild assumptions, we provide approximate guarantees for the effectiveness of this method, with experimental analysis supported, which will be elaborated in the Appendix. 

\subsubsection{Step 4: Matrix Recovery}
It is necessary to provide an outline of \textit{unlabeled sensing} for completeness. In a nutshell, unlabeled sensing methods \cite{upca,slawski2019linear} are proposed for solving linear equation systems with unordered measurements:
\begin{equation}
 \mathbf{y} = \boldsymbol{\pi} \mathbf{Ux}   
\end{equation}
 where $\boldsymbol{\pi}$ is an unknown permutation matrix, with the knowledge of $\mathbf{y}$ and $\mathbf{U}$ \cite{unlabeledsensing}. Different unlabeled sensing methods are brought into practice according to the rank of the basis $\mathbf{U}$ and the type of shuffling (partially shuffled or fully shuffled).

During the matrix recovery step (step \ref{step4:matrix-recvery} in Algorithm \ref{alg:algorithm_4_stage}), we employ an unlabeled sensing method to iteratively recover each outlier $\tilde{\mathbf{y}}_t$ with its corresponding basis $\widehat{\mathbf{U}}_t$. Alternatively, matrix recovery methods like robust PCA \cite{rpca}, LRR \cite{liu2012robust}, and RKPCA \cite{fan2019exactly} could also be utilized to recover $[\widehat{\mathbf{X}}_t, \widetilde{\mathbf{Y}}_t ]$ associated with subspace $\mathcal{S}_t$.

\section{Theory for Outlier Classification}

In this section, we provide a theoretical guarantee for Algorithm \ref{alg:solve_partial_lsr}. To begin with, we have the following assumptions.

\begin{assumption}\label{Assumption:1}
The variables 
$\mathbf{(\tilde{y} - \hat{\tilde{y}})}_1, \ldots, \mathbf{(\tilde{y} - \hat{\tilde{y}})}_{M_1}$ 
are independent and identically distributed (i.i.d.) following a distribution denoted by $\xi$. Similarly, the variables $\mathbf{(\tilde{y} - \hat{\tilde{y}})}_{M_1 + 1}, \ldots, \mathbf{(\tilde{y} - \hat{\tilde{y}})}_{M}$ 
are i.i.d. following a distribution denoted by $\eta$. Both $\xi$ and $\eta$ belong to the same bell-shaped distribution cluster, with a mean value 
$\mu_\xi = \mu_\eta \triangleq \mu = 0$, 
differing only in their variances 
$\sigma_\xi^2 \neq \sigma_\eta^2$, 
which means their cumulative distribution functions satisfy:
\begin{align}\label{CDF: initial}
    F_{\xi}(\sigma_\xi x)=F_{\eta}(\sigma_\eta x)\triangleq F(x),
\end{align}
where $F(x)$ is the cdf of their normalized distribution with variance $\int_{\mathbb{R}} x^2\,dF(x) = 1$.
\end{assumption}

\begin{assumption} \label{Assumption:2}
For the sake of brevity, we assume that the bell-shaped distribution in Assumption \ref{Assumption:1} is given by $F(x) = \Phi(x)$, where $\Phi(x)$ denotes the cdf of the standard Gaussian distribution.
\end{assumption}
We defer the detailed discussion on the assumptions to Appendix. Now we present the following theorem:
\begin{theorem}\label{Theorem: Initial}
Under Assumptions \ref{Assumption:1} and \ref{Assumption:2}, and without loss of generality, let $\mathcal{A} \triangleq \{i \in \mathbb{Z}^{+}: 1 \leq i \leq M_1\}$ represent the unshuffled indices and $\mathcal{O} \triangleq \{i \in \mathbb{Z}^{+}: M_1 + 1 \leq i \leq M\}$ represent the shuffled indices of $\mathbf{y}$. Thus $M_1=\#(\mathcal{A})$, $M_2=\#(\mathcal{O})=M-M_1$.  Define
\begin{equation}
\hat{j} = \operatorname{argmax}_j \left| \left(\mathbf{y} - \mathbf{\hat{\tilde{y}}}\right)_j \right|.
\end{equation}
Then, approximately,
\begin{align}\label{Final Probability}
\Pr\left(\hat{j} \in \mathcal{O}\right) 
\approx 
2\Phi \left(
\frac{\sigma_\eta\rho(M_2) - \sigma_\xi\rho(M_1)}{\sqrt{\sigma_\eta^2\psi^2(M_2) + \sigma_\xi^2\psi^2(M_1)}} 
\right)-1,
\end{align}
where
\begin{align}
    \rho(m) &= F^{-1}(1-\frac{1}{m}) = \Phi^{-1}(1-\frac{1}{m}),\\
    \psi(m) &= \frac{1}{m\cdot f(\rho(m))} = \frac{1}{m\cdot \phi(\rho(m))},
\end{align}
with $f$ (or $\phi$) being the pdf corresponding to $F$ (or $\Phi$) and approximated estimation of $\sigma_\xi^2, \sigma_\eta^2$ as follows:
\begin{align}
\begin{cases}
\Pr\left(
    \sigma_\xi^2 < C\left(\frac{r(M-r)M_2}{M^2(M-1)(M+2)}\right)
    \right) 
> 
1-\delta 
\\
\mathbb{E}(\sigma_\xi^2)
\leq
\frac{M_2r(M - r)}{M^2(M - 1)(M + 2)} + \frac{\sqrt{6}M_2r^{1/2}(M - r)^{3/2}}{M^2(M - 1)(M + 2)^{3/2}}+\frac{M_2^2}{M^3}
\\
\mathbb{E}(\sigma_\eta^2)
\geq
\frac{2}{M}
-
\frac{2[M_2M+M-4](M-r)}{M^2(M-1)(M+2)}
+
\frac{
\Gamma_r\left(\frac{r}{2}+\frac{3}{r}\right)\Gamma_r\left(\frac{M}{2}\right)
}{
M_2\Gamma_r\left(\frac{M}{2} + \frac{3}{r}\right)\Gamma_r\left(\frac{r}{2}\right)
}
\\
\mathbb{E}(\sigma_\eta^2)
\approx
\frac{(2M - M_2)}{M^2}
\left[ \left(\frac{r}{M} - 1\right)^2 + \frac{2r(M - r)}{M^2(M + 2)} + \frac{(M_2 - 1)r(M - r)}{M(M - 1)(M + 2)} \right]
\end{cases}
\label{eq:final variance of target expression}
\end{align}
where $C \leq 2+6(1+\sqrt{2})\sqrt{\frac{M-r}{rM_2M}}$, $\delta\approx 0.0054$, and multivariate gamma function
\begin{align*}
    \Gamma_r(x) = \pi^{\frac{r(r-1)}{4}} \prod_{i=1}^r \Gamma\left(x - \frac{i-1}{2}\right).
\end{align*}
\end{theorem}

There are some calculation issues for $\sigma_\xi^2$ and $\sigma_\eta^2$, which will be detailed in Appendix. Anyway, Theorem \ref{Theorem: Initial} ensures that Algorithm \ref{alg:solve_partial_lsr} successfully recovers the subspace when initialized with the ground truth subspace basis. Specifically, the shuffled ratio $\frac{M_2}{M}$ in the retaining vector $\bm{\nu}^{(i)}$ decreases rapidly as the iteration index $i$ increases. Consequently, $\bm{\nu}^{(i)}$ quickly aligns with the ground-truth retaining subspace $\mathcal{S}^{(i)}$ in terms of cosine distance, eventually approaching zero. Intuitively, this can be understood as the following: with the ratio of retained entries being no more than $\gamma$ (as defined in Algorithm \ref{alg:solve_partial_lsr}), we can confidently assert that most of the shuffled entries have been removed. As a result, the retaining vector $\bm{\nu}^{(iter)}$ closely approximates the ground-truth retaining subspace $\mathcal{S}_{gt}^{(iter)}$.

Conversely, if the process starts with an entirely incorrect subspace, it is analogous to a situation where the data points have been completely shuffled, as discussed in the Appendix. In such a case, the entries are eliminated in a seemingly random fashion, making it impossible to distinguish between shuffled and unshuffled entries. Consequently, Algorithm \ref{alg:solve_partial_lsr} would fail, as corroborated by our theoretical analysis. In this scenario, the retaining vector $\bm{\nu}^{(i)}$ will not converge towards the incorrect subspace $\mathcal{S}_{wrong}^{(i)}$ at the same rate as when the correct subspace is used. This is because the retaining vector $\bm{\nu}^{(iter)}$ bears little correlation with $\mathcal{S}_{wrong}^{(iter)}$, given that the retained entries belong to an unrelated subspace.
Thus, by selecting an appropriate stopping criterion $\gamma$ in Algorithm \ref{alg:solve_partial_lsr}, we can effectively differentiate the ground truth subspace label by comparing the cosine distances between each retaining vector $\bm{\nu}_k^{(i)}$ and its corresponding subspace $\mathcal{S}_k^{(i)}$, for $k = 1, \ldots, L$.

\section{Experimental Evaluation}
\begin{table}[h]
\centering
\begin{tabularx}{230pt}{p{0.85cm}|X} 
\toprule
(L, p) & Permutation Error Ratio$(\%)$ \\
\midrule
(2, 2) &
\begin{tabular}[c]{@{}l@{}}
\textbf{MRUC-S}: 0.0 ± 0.0 ([0, 0]) \\
\textbf{MRUC}: 12.0 ± 15.5 ([0, 30])
\end{tabular} \\
\midrule
(3, 2) &
\begin{tabular}[c]{@{}l@{}}
\textbf{MRUC-S}: 0.0 ± 0.0 ([0, 0]) \\
\textbf{MRUC}: 30.0 ± 0.0 ([30, 30])
\end{tabular} \\
\midrule
(3, 3) &
\begin{tabular}[c]{@{}l@{}}
\textbf{MRUC-S}: 13.3 ± 16.6 ([0, 33.3]) \\
\textbf{MRUC}: 55.0 ± 1.4 ([53.3, 58.3])
\end{tabular} \\
\midrule
(3, 5) &
\begin{tabular}[c]{@{}l@{}}
\textbf{MRUC-S}: 17.0 ± 25.9 ([0, 65]) \\
\textbf{MRUC}: 71.5 ± 13.7 ([56.7, 91.7])
\end{tabular} \\
\midrule
(5, 5) &
\begin{tabular}[c]{@{}l@{}}
\textbf{MRUC-S}: 23.0 ± 17.9 ([0, 50]) \\
\textbf{MRUC}: 73.2 ± 0.6 ([73, 75])
\end{tabular} \\
\bottomrule
\end{tabularx}
\caption{Performance comparison of MRUC and MRUC-S. Values represent the permutation error ratio (\%), which is the average normalized Hamming distance between predicted and true permutation matrices. Metrics are reported as mean ± std ([min, max]) over multiple random initializations.}
\label{tab:MRUC synthetic experiments}
\end{table}

\begin{table*}[htbp]
\centering
\begin{tabular}{@{}ccccc@{}}
\toprule
\quad \textbf{Mean (Median)}  & 
\textbf{CE}\textsubscript{gt} & \textbf{CE}\textsubscript{recon} & 
\textbf{UOratio} & 
\textbf{SCerr} \\ 
\midrule
2 subjects  & 0.0000 (0.0000)& 0.0000 (0.0000)& 0.1042 (0.1095) & 0.0137 (0.0053) \\
3 subjects  & 0.0000 (0.0000)& 0.0110 (0.0105)& 0.0488 (0.0525) & 0.0106 (0.0072) \\
5 subjects  & 0.0168 (0.0190)& 0.0343 (0.0325)& 0.0333 (0.0250) & 0.0493 (0.0304) \\
8 subjects  & 0.0245 (0.0230)& 0.0595 (0.0630)& 0.0157 (0.0120) & 0.0959 (0.1007) \\
10 subjects & 0.0250 (0.0250)& 0.1002 (0.0845)& 0.0107 (0.0130) & 0.1336 (0.1563) \\
12 subjects & 0.0268 (0.0285)& 0.1025 (0.1030)& 0.0093 (0.0100) & 0.1409 (0.1892) \\ \bottomrule
\end{tabular}
\caption{Performance of Experiment on Extended YaleB Dataset}
\label{tab:face_CE}
\end{table*}

Before presenting the results of experiments conducted on synthetic and real-world datasets, it is important to clarify the evaluation metrics used to demonstrate the effectiveness of our approach.

\textbf{Outlier Classification Error} involves two metrics, \(\textbf{CE\textsubscript{gt}}\) and \(\textbf{CE\textsubscript{recon}}\), that measure outlier classification accuracy. \(\textbf{CE\textsubscript{gt}}\) is calculated using the ground truth subspaces, \(\mathbf{U}_1, \ldots, \mathbf{U}_L\), while \(\textbf{CE\textsubscript{recon}}\) uses reconstructed bases \(\widehat{\mathbf{U}}_1, \ldots, \widehat{\mathbf{U}}_L\). 
Then:
\begin{align*}
    &\textbf{CE\textsubscript{gt}} = \tfrac{\#\left(\text{Misclassified Outliers}\right)}{\#\left(\text{Outliers}\right)} \\
    &\textbf{CE\textsubscript{recon}} = \tfrac{\#\left(\text{Misclassified Detected Outliers}\right)}{\#\left(\text{Detected Outliers}\right)}
\end{align*}

\textbf{Matrix Recovery Error} is assessed by \(\textbf{RE\textsubscript{gt}}\) and \(\textbf{RE\textsubscript{recon}}\), which measure the accuracy of recovering outlier columns using normalized Frobenius norms. Specifically:
\begin{align*}
    &\textbf{RE\textsubscript{gt}} = \tfrac{\|\operatorname{Proj}_\mathcal{S}(\widehat{\mathbf{Y}}) - \mathbf{Y}\|_F}{\|\mathbf{Y}\|_F} \\
    &\textbf{RE\textsubscript{recon}} = \tfrac{\|\operatorname{Proj}_\mathcal{S}(\widehat{\mathbf{Y}}_d) - \mathbf{Y}_d\|_F}{\|\mathbf{Y}_d\|_F}
\end{align*}
where \(\widehat{\mathbf{Y}}\) and \(\widehat{\mathbf{Y}}_d\) are recovered outliers using ground truth and detected bases, respectively, and \(\mathbf{Y}\) and \(\mathbf{Y}_d\) are the corresponding ground truth outliers. \(\operatorname{Proj}_\mathcal{S}(\cdot)\) denotes projection onto the subspace.

\textbf{Auxiliary Metrics} include \textbf{UOratio} and \textbf{SCerr}:
\begin{align*}
    &\textbf{UOratio} = \tfrac{\#\left(\text{Undetected Outliers}\right)}{\#\left(\text{Outliers}\right)} \\
    &\textbf{SCerr} = \tfrac{\#\left(\text{Misclassified Detected Inliers}\right)}{\#\left(\text{Detected Inliers}\right)}
\end{align*}
\textbf{UOratio} evaluates the undetected ratio of outlier detection, while \textbf{SCerr} assesses subspace clustering error, both influencing \(\textbf{CE\textsubscript{recon}}\) and \(\textbf{RE\textsubscript{recon}}\).

\subsection{Experiment on Synthetic Data}
\begin{figure}[t]
    \centering
    \begin{subfigure}[t]{0.15\textwidth}
        \centering
        \includegraphics[width=\textwidth]{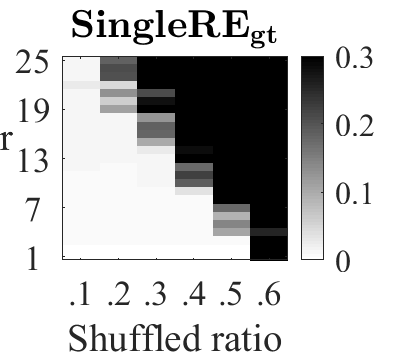}
        \caption{Recovery Error with GT Single Subspace}
        \label{subfig:Recovery_Error_For_Single_Subspace}
    \end{subfigure}
    \hfill
    \begin{subfigure}[t]{0.15\textwidth}
        \centering
        \includegraphics[width=\textwidth]{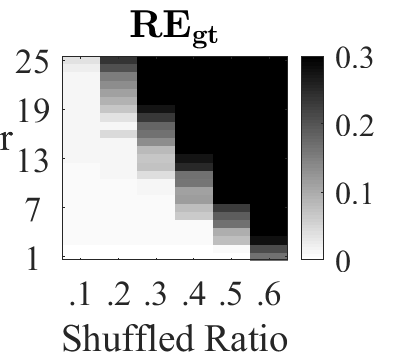}
        \caption{Recovery Error with GT Information}
        \label{subfig:Recovery_Error_with_GT_Information}
    \end{subfigure}
    \hfill
    \begin{subfigure}[t]{0.15\textwidth}
        \centering
        \includegraphics[width=\textwidth]{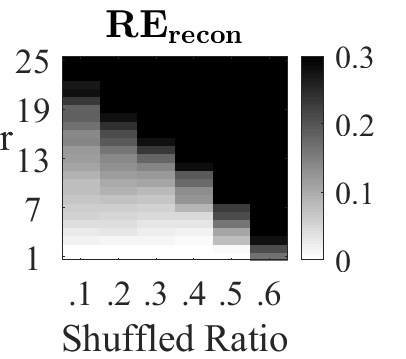}
        \caption{Recovery Error with Reconstructed Information}
        \label{subfig:Recovery_Error_with_Recon_Information}
    \end{subfigure}
    \vspace{0.1cm} 
    \begin{subfigure}[t]{0.15\textwidth}
        \centering
        \includegraphics[width=\textwidth]{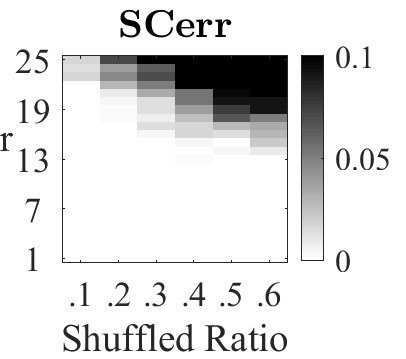}
        \caption{Subspace Clustering Error}
        \label{subfig:Missrate_In}
    \end{subfigure}
    \hfill
    \begin{subfigure}[t]{0.15\textwidth}
        \centering
        \includegraphics[width=\textwidth]{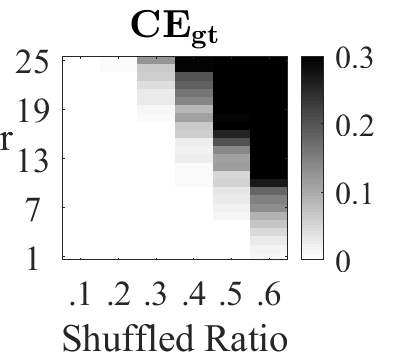}
        \caption{Outlier Classification Error with GT Information}
        \label{subfig:Missrate_Out_with_GT_Information}
    \end{subfigure}
    \hfill
    \begin{subfigure}[t]{0.15\textwidth}
        \centering
        \includegraphics[width=\textwidth]{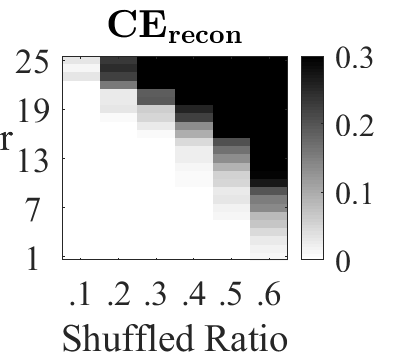}
        \caption{Outlier Classification Error with Reconstructed Information}
        \label{subfig:Missrate_Out_with_Recon_Information}
    \end{subfigure}
    \caption{Performance of PMSDR (Algorithm \ref{alg:algorithm_4_stage}) on Synthetic Data. Experiments are conducted for sparse permutations with shuffled ratios up to $0.6$ and subspace dimensions up to $25$ with the ambient space dimension being $50$. In multi-subspace cases, subspace settings are $2, 3, 5, 8, 10$, and the median error is plotted.}
    \label{fig:Synthetic Experiment}
\end{figure}

\begin{figure}[t]
    \centering
    \begin{subfigure}[t]{0.15\textwidth}
        \centering
        \includegraphics[width=\textwidth]{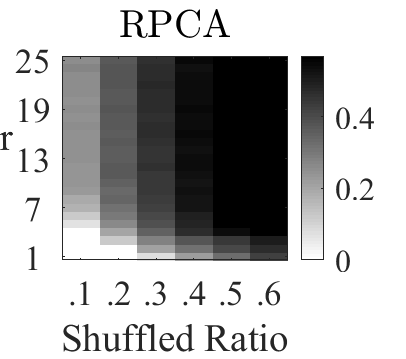}
        \caption{Recovery Error for RPCA}
        \label{subfig:Recovery_Error_For_RPCA}
    \end{subfigure}
    \hfill
    \begin{subfigure}[t]{0.15\textwidth}
        \centering
        \includegraphics[width=\textwidth]{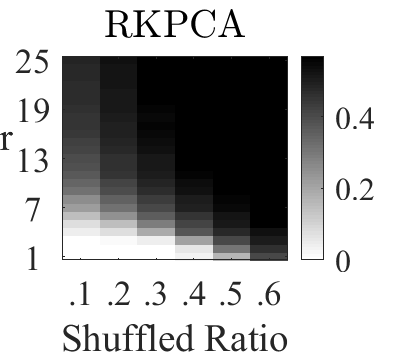}
        \caption{Recovery Error for RKPCA}
        \label{subfig:Recovery_Error_For_RKPCA}
    \end{subfigure}
    \hfill
    \begin{subfigure}[t]{0.15\textwidth}
        \centering
        \includegraphics[width=\textwidth]{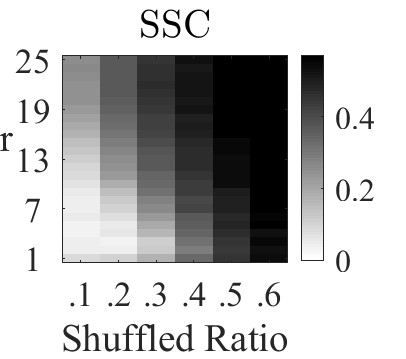}
        \caption{Recovery Error for SSC}
        \label{subfig:Recovery_Error_For_SSC}
    \end{subfigure}
    \vskip\baselineskip 
    \begin{subfigure}[t]{0.15\textwidth}
        \centering
        \includegraphics[width=\textwidth]{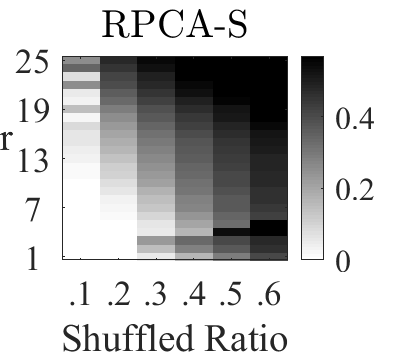}
        \caption{Recovery Error for RPCA-S}
        \label{subfig:Recovery_Error_For_PMSDR_RPCA}
    \end{subfigure}
    \hfill
    \begin{subfigure}[t]{0.15\textwidth}
        \centering
        \includegraphics[width=\textwidth]{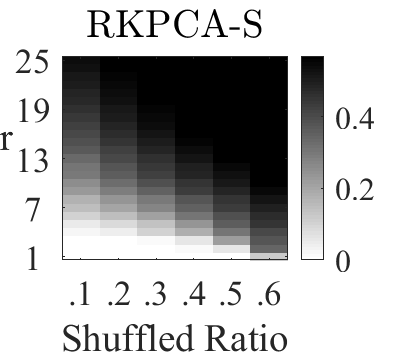}
        \caption{Recovery Error for RKPCA-S}
        \label{subfig:Recovery_Error_For_PMSDR_RKPCA}
    \end{subfigure}
    \hfill
    \begin{subfigure}[t]{0.15\textwidth}
        \centering
        \includegraphics[width=\textwidth]{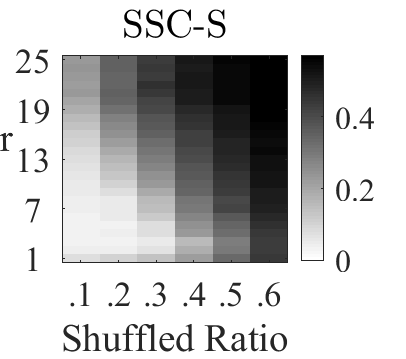}
        \caption{Recovery Error for SSC-S}
        \label{subfig:Recovery_Error_For_PMSDR_SSC}
    \end{subfigure}
    \caption{Synthetic experiments on RPCA, RKPCA, SSC, and their PMSDR-augmented versions. All experiments are conducted for sparse permutations with shuffled ratios no greater than 0.6 and subspace dimensions no greater than 50\% of the ambient space dimension, which is 50. The number of subspaces is fixed to 5.}
    \label{fig:Synthetic_Experiment_Comparison}
\end{figure}

We conduct three experiments: (1) analyzing Algorithm \ref{alg:algorithm_4_stage}, (2) comparing it with RPCA \cite{rpca}, RKPCA \cite{fan2019exactly}, and SSC \cite{SSC_a}, and (3) evaluating the impact of MRUC \cite{MRUC} on permutation matrix recovery when Algorithm \ref{alg:algorithm_4_stage} is augmented for multiple subspaces.

In the first experiment, the ambient dimension $M$ is $50$, each subspace group has $120$ samples, and the outlier proportion is $60\%$. The number of groups is $[2, 3, 5, 8, 10]$. The subspace rank varies from $1$ to $25$, and the shuffled ratio from $0.1$ to $0.6$, with a noise level of $40$ dB. The median error across all settings is recorded.
Figure \ref{fig:Synthetic Experiment} compares the estimation error of our method with single subspace results. Figure \ref{fig:Synthetic Experiment}(a) shows UPCA \cite{upca} on a single subspace as a baseline, while (b) and (c) show multi-subspace results, demonstrating minimal performance loss even with reconstructed information. Figures (e) and (f) highlight the robust outlier classification.

In the second experiment, we compare vanilla methods with their PMSDR-augmented versions, with the number of subspaces fixed at $5$. Figure \ref{fig:Synthetic_Experiment_Comparison} shows significant enhancement in RPCA and SSC, and slight improvement for RKPCA when augmented.
The third experiment compares PMSDR-augmented MRUC-S with MRUC using Hamming distance as the evaluation metric. The experimental setup includes $M=20$, $r=2$, and a shuffled ratio of $0.5$. Table \ref{tab:MRUC synthetic experiments} demonstrates the superior performance of PMSDR, highlighting the robustness of Algorithm \ref{alg:algorithm_4_stage}, even under varying subspace configurations.

In summary, our method effectively bridges multi- and single-subspace cases in low-rank scenarios, generalizing single-subspace recovery to multi-subspace contexts.

\begin{figure}[htbp]
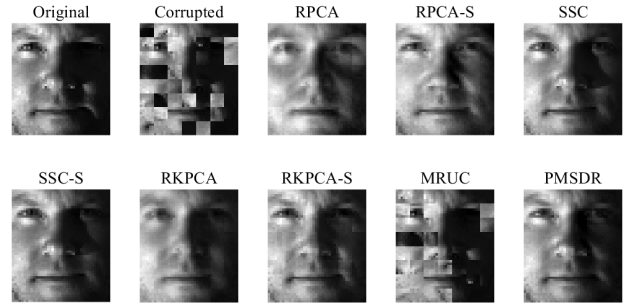

    \centering
    \foreach \folder in {Original, Corrupted, RPCA, RPCA-S, SSC} 
    {
        \begin{subfigure}[t]{0.098\textwidth} 
            \graphicspath{{\folder/}}
            \includegraphics[width=\textwidth]{\folder_23_3.png} 
            \caption*{}
        \end{subfigure}
        \hspace{-3mm}
    }
    
    \vspace{-7mm} 
    \foreach \folder in {SSC-S, RKPCA, RKPCA-S, MRUC, PMSDR} 
    {
        \begin{subfigure}[t]{0.098\textwidth} 
            \graphicspath{{\folder/}}
            \includegraphics[width=\textwidth]{\folder_23_3.png} 
            \caption*{}
        \end{subfigure}
        \hspace{-3mm}
    }

    \vspace{-7mm}
    \caption{Experimental results showing a subset of the image recovery experiments. The complete set of results is in Appendix \ref{Appendix: Face}.}
    \label{fig:face_images_compressed}
\end{figure}

\subsection{Experiment on Face Images}

\begin{table}[h]
\centering
\begin{tabular}{l|l|ccc}
\toprule
\textbf{Subspaces} & \textbf{Metric} & \textbf{Mean} & \textbf{Median} \\ 
\midrule
\multirow{5}{*}{2 subspaces} & \textbf{CE}\textsubscript{gt} 
                                    & $0.011$ & $0.000$ \\ 
& \textbf{CE}\textsubscript{recon}  & $0.040$ & $0.007$ \\ 
& \textbf{RE}\textsubscript{gt}     & $0.017$ & $0.005$ \\ 
& \textbf{RE}\textsubscript{recon}  & $0.021$ & $0.013$ \\ 
& \textbf{SCerr}                    & $0.027$ & $0.000$ \\ 
\midrule
\multirow{5}{*}{3 subspaces} & \textbf{CE}\textsubscript{gt} 
                                    & $0.018$ & $0.008$ \\ 
& \textbf{CE}\textsubscript{recon}  & $0.093$ & $0.023$ \\ 
& \textbf{RE}\textsubscript{gt}     & $0.014$ & $0.007$ \\ 
& \textbf{RE}\textsubscript{recon}  & $0.021$ & $0.012$ \\ 
& \textbf{SCerr}                    & $0.043$ & $0.004$ \\ 
\bottomrule
\end{tabular}
\caption{PMSDR Performance on Hopkins-155}
\label{tab:hopkins_PMSDR}
\end{table}

\begin{table}[htbp]
\centering
\begin{tabular}{l|c|c}
\hline
\textbf{Method} & \textbf{2 subspaces} & \textbf{3 subspaces} \\
\hline
\multicolumn{3}{c}{\textbf{Median RE (Mean RE)}} \\
\hline
PMSDR             & $0.013 \, (0.021)$ & $0.012 \, (0.021)$ \\
PMSDR\textsubscript{gt}          & $\mathbf{0.005 \, (0.017)}$ & $\mathbf{0.007 \, (0.014)}$ \\
\hline
RPCA              & $0.046 \, (0.052)$ & $0.047 \, (0.055)$ \\
RPCA-S            & $0.040 \, (0.045)$ & $0.043 \, (0.047)$ \\
RPCA\textsubscript{gt} -S        & $\mathbf{0.033 \, (0.040)}$ & $\mathbf{0.030 \, (0.036)}$ \\
\hline
RKPCA             & $\mathbf{0.009 \, (0.014)}$ & $\mathbf{0.008 \, (0.012)}$ \\
RKPCA-S           & $0.053 \, (0.064)$ & $0.047 \, (0.058)$ \\
RKPCA\textsubscript{gt} -S       & $0.039 \, (0.057)$ & $0.036 \, (0.049)$ \\ 
\hline
SSC               & $0.087 \, (0.090)$ & $0.100 \, (0.104)$ \\
SSC-S             & $0.091 \, (0.095)$ & $0.101 \, (0.102)$ \\
SSC\textsubscript{gt}-S          & $\mathbf{0.079 \, (0.084)}$ & $\mathbf{0.089 \, (0.090)}$ \\
\hline
MRUC              & $0.066 \, (0.084)$ & $0.085 \, (0.091)$    \\
MRUC-S            & $0.048 \, (0.074)$ & $0.056 \, (0.078)$    \\
MRUC\textsubscript{gt}-S         & $\mathbf{0.048 \, (0.070)}$ & $\mathbf{0.056 \, (0.073)}$    \\
\hline
\end{tabular}
\caption{Matrix Recovery Error for Comparison Experiments on Hopkins-155.}
\label{tab:hopkins_RPCA}
\end{table}


We applied our algorithm to the Extended Yale B dataset \cite{yaleb}, which includes $38$ subjects, each with $64$ downsampled face images of size $48 \times 42$ ($M=2016$).

In the first experiment, we selected $10$ subjects and corrupted $19$ images per group by shuffling $40\%$ of the pixels. We compared our PMSDR method with RPCA \cite{Robust_PCA}, SSC \cite{SSC_a}, and RKPCA \cite{fan2019exactly}, as well as their PMSDR-augmented counterparts (RPCA-S, SSC-S, RKPCA-S). The results, depicted in Figure \ref{fig:face_images}, demonstrate that PMSDR substantially improves matrix recovery and outlier correction. A subset of these results is presented in Figure \ref{fig:face_images_compressed}, with the complete set available in Appendix \ref{Appendix: Face}.

In the second experiment, we varied the number of subspaces $L$ from $2$ to $12$ and repeated the corruption process. The results, summarized in Table \ref{tab:face_CE}, indicate a strong performance, particularly when the ground truth is known. Specifically, \(\textbf{CE\textsubscript{gt}}\) is lower than \(\textbf{CE\textsubscript{recon}}\), due to increased subspace clustering errors \textbf{SCerr}. However, these findings highlight the robustness and adaptability of our algorithm.

\subsection{Experiment on Motion Segmentations}

We evaluated our algorithm on the Hopkins-155 database, which includes $117$ sequences with $2$ subspaces, $35$ with $3$ subspaces and $1$ with $5$ subspaces. Each sequence lies in a $4$-dimensional subspace \cite{hop_dim4_1,hop_dim4_2}. The shuffled and outlier ratios are both $0.4$, with a fixed subspace dimension of $4$. Data are mapped from $3$D to $2$D for better representation, and concatenated over frames. After preprocessing, the dimension of the ambient space ranges from $40$ to $70$.

We apply our 4-stage pipeline (PMSDR) and compare its performance against RPCA, RKPCA, SSC, and MRUC, along with their PMSDR-augmented variants, denoted by appending the suffix `-S'. Additionally, we investigate the impact of using ground-truth information for matrix recovery. The regularization parameters in RPCA, RKPCA, and SSC are tuned accordingly. Results in Table \ref{tab:hopkins_PMSDR} show PMSDR's robustness in matrix recovery and outlier classification, even without ground truth information. Table \ref{tab:hopkins_RPCA} highlights further improvements when combining these methods with our pipeline, except for RKPCA pairs, which still performs well as for the PMSDR-augmented version.

\subsection{Experiment on Data Re-identification}

We evaluated the proposed PMSDR pipeline alongside the RPCA, SSC, RKPCA, and MRUC methods on real-world educational and medical records, simulating a privacy protection scenario similar to \cite{upca}. The first dataset, described in \cite{record_linkage}, contains a matrix $M_{score} \in \mathbb{R}^{707\times 14}$ with scores of $707$ students across $14$ tests. To anonymize, the last $7$ columns were randomly permuted, with shuffled ratios from $0.1$ to $1$. The second dataset from \cite{Breast_tumor} involves a matrix $M_{tumor} \in \mathbb{R}^{357 \times 30}$, representing $357$ patients with $30$ features. Here, $50\%$ of the columns were permuted with similar shuffled ratios. Both matrices were normalized, and our method was applied with a subspace dimension of $3$ as in \cite{upca}. We assumed prior knowledge of the outlier ratio during the detection phase (Step \ref{Step 1: Outlier Detection.} in Algorithm \ref{alg:algorithm_4_stage}).

Given the lack of inherent multi-subspace scenarios in these datasets, no ground truth subspace information is available. Thus, even single subspace methods can perform reasonably well. We compared our PMSDR with RPCA ($\lambda$ optimized over $0.1:0.05:0.95$ and $\mu = 10\lambda$), RKPCA ($\lambda$ optimized similarly), SSC ($\alpha$ optimized over $[5, 20, 100, 200, 500, 1000]$), and MRUC with the best initialization. UPCA \cite{upca}, a strong single-subspace recovery method, was used as a baseline.

In PMSDR, we hypothesized subspaces $L=1, 2, 3$ despite the lack of natural subspaces. Figures \ref{fig:de-anonymization-educational} and \ref{fig:de-anonymization-medical} show that PMSDR outperforms UPCA in most cases, with slightly worse performance in low shuffled ratios due to possible undetected outliers affecting the basis of the subspace. This suggests that Algorithm \ref{alg:algorithm_4_stage} can uncover more hidden information, improving robustness and performance.

\begin{figure}[t]
    \centering
    \begin{minipage}[b]{0.48\textwidth}
        \includegraphics[width=\textwidth]{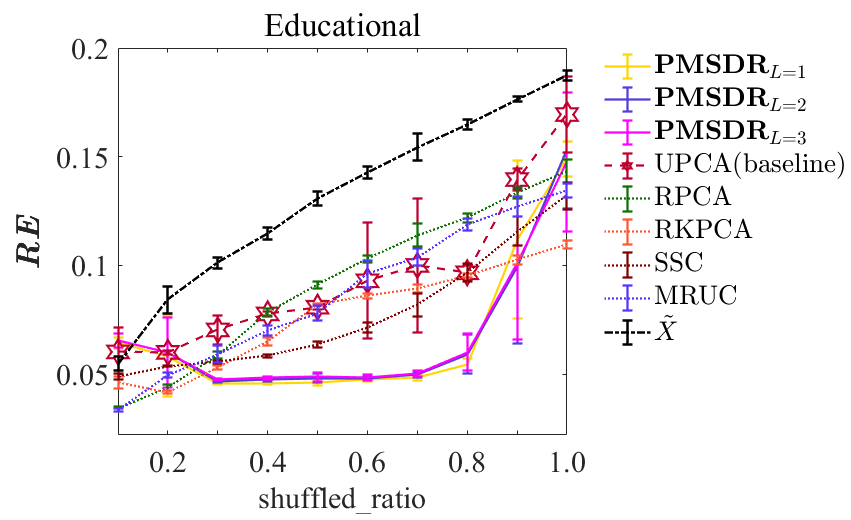}
        \caption{De-Anonymization Experiments for Educational Data comparing Algorithm \ref{alg:algorithm_4_stage}, UPCA \cite{upca}, and three other methods (RPCA, RKPCA, SSC and MRUC). The output of Algorithm \ref{alg:algorithm_4_stage} is denoted as \textbf{PMSDR}\textsubscript{$L=k$}, where the number of groups $k$ is set to \{1, 2, 3\}.}
        \label{fig:de-anonymization-educational}
    \end{minipage}
    \hfill
    \begin{minipage}[b]{0.48\textwidth}
        \includegraphics[width=\textwidth]{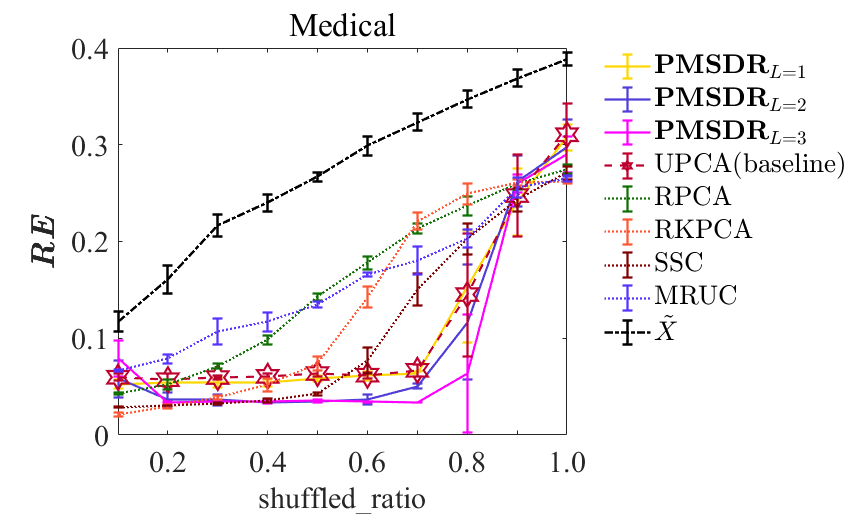}
        \caption{De-Anonymization Experiments for Medical Data comparing Algorithm \ref{alg:algorithm_4_stage}, UPCA \cite{upca}, and three other methods (RPCA, RKPCA, SSC and MRUC). The output of Algorithm \ref{alg:algorithm_4_stage} is denoted as \textbf{PMSDR}\textsubscript{$L=k$}, where the number of groups $k$ is set to \{1, 2, 3\}.}
        \label{fig:de-anonymization-medical}
    \end{minipage}

    \label{fig:de-anonymization}
\end{figure}


\section{Conclusion and Future Directions}

Our algorithm extended UPCA \cite{upca} by recovering corrupted data across multiple subspaces. For $L > 1$, complexity increases due to outlier matching. Instead of the robust PCA, we use outlier detection and subspace clustering to estimate bases. While our outlier classification (Step 3 in Algorithm \ref{alg:algorithm_4_stage}, which detailed in Algorithm \ref{alg:solve_partial_lsr}) performs well, it depends on accurate basis estimation, revealing potential areas for improvement.

The framework is flexible, integrating other methods, but currently handles only linear/affine cases. Extending it to non-linear contexts requires further research. The algorithm is also limited to partially shuffled data, a restriction future work should address. Applying permutation recovery methods like MRUC \cite{MRUC} as preprocessing could transform fully shuffled data into a partially shuffled state, enabling our PMSDR pipeline to function effectively.

In conclusion, our method enhances recovery across multiple subspaces, but further research is needed to improve basis estimation, handle non-linear scenarios, and overcome the partially shuffled data limitation.

\bigskip

\section*{Acknowledgments}
This work was supported by the Shenzhen Science and Technology Program under Grant No.JCYJ20210324130208022 (Fundamental Algorithms of Natural Language Understanding for Chinese Medical Text Processing) and the Youth Program 62106211 of the National Natural Science Foundation of China.

\bibliography{aaai25}

\clearpage
\onecolumn
\appendix

\section{Computational Complexity Analysis}
We analyze the time complexity of our PMSDR theoretically. The time complexities in Outlier Detection, Subspace Reconstruction, Outlier Classification, and Matrix Recovery are $\mathcal{O}(T_1MN^2)$, $\mathcal{O}(T_2MN^2+L\min(M(N/L)^2,(N/L)M^2))$, $\mathcal{O}(LT_3(rM^2+r^3))$, and $\mathcal{O}(N_Yr^2M^2)$ respectively. Suppose $N/L>M$, the total time complexity is $\mathcal{O}((T_1+T_2)MN^2+NM^2+T_3LrM^2+N_Yr^2M^2)$. This complexity can be further reduced using sparse matrix multiplication and truncated SVD. 

We also compare the time costs (second) of our method and the latest competitors MRUC and RKPCA on the synthetic data with different number of subspaces. As shown in the table, our method is more efficient.

\begin{table}[h!]
    \centering
    \begin{tabular}{c|c|c|c}
\hline
\text{No. of subspaces} & \text{PMSDR} & \text{RKPCA} & \text{MRUC} \\ 
\hline
2& $1.5$ & $\mathbf{1.2}$ & $679.4$ \\
5& $\mathbf{4.3}$ & $9.9$ & $698.8$ \\
10& $\mathbf{15.7}$ & $54.5$ & $1072.3$\\ 
\hline
    \end{tabular}
    \caption{Time cost (second) comparison on synthetic data.}
    \label{tab:my_label}
\end{table}

\section{More about Experiments}
\subsection{Face Experiments} \label{Appendix: Face}

\begin{figure*}[htbp]
    \centering
    \foreach \folder in {Original, Corrupted, RPCA, RPCA-S, SSC, SSC-S, RKPCA, RKPCA-S, MRUC, PMSDR}
    {
        \begin{minipage} {\textwidth}
            \foreach \i in {1,...,10} {
                \begin{subfigure}[t]{0.093\textwidth}
                    \graphicspath{{\folder/}}
                    \includegraphics[width=\textwidth]{\folder_23_\i.png} 
                    \caption*{}
                \end{subfigure}
            }
            \vspace{-8mm}
        \end{minipage}
    }
    \caption{Experimental results presenting the ground truth images (Original), the corrupted examples (Corrupted), and the recovered images using robust PCA (RPCA) \cite{rpca}, sparse subspace clustering (SSC) \cite{SSC_a}, and robust kernel PCA (RKPCA) \cite{fan2019exactly}. Additionally, the results include the images recovered by the augmented versions of the aforementioned methods processed with our PMSDR pipeline (RPCA-S, SSC-S, RKPCA-S), demonstrating the practical application of these methods in the matrix recovery step (Algorithm \ref{alg:algorithm_4_stage}, Step \ref{step4:matrix-recvery}). Furthermore, the results from permutation-targeted recovery method MRUC are presented. Finally, the results from our pipeline using LSRF \cite{upca} as the default outlier recovery method (PMSDR) are showcased. As illustrated, the original methods struggle with recovery quality under multi-subspace conditions due to hybrid subspace information. In contrast, the PMSDR augmented versions significantly enhance performance by reinforcing the subspace structure with Algorithm \ref{alg:algorithm_4_stage}.
}
    \label{fig:face_images}
\end{figure*}

We evaluate our proposed method on the Extended Yale B dataset \cite{yaleb}, which includes $38$ subjects with $64$ images each, captured under varying lighting conditions. Images are downsampled to $48 \times 42$, yielding a vectorized dimension of $M=2016$. Two experimental setups are employed.

In the first setup, we select $10$ subjects, corrupting $19$ images per subject by shuffling $40\%$ of the pixels. The subspace dimension $r$ is fixed at $8$. We compare our \textbf{PMSDR} method against \textbf{RPCA} \cite{Robust_PCA}, \textbf{SSC} \cite{SSC_a}, and \textbf{RKPCA} \cite{fan2019exactly}, along with their augmented versions (\textbf{RPCA-S}, \textbf{SSC-S}, \textbf{RKPCA-S}). Regularization parameters are set as follows: $\lambda = \frac{1}{\sqrt{\max(\#(\text{rows}),\#(\text{columns}))}}$ for RPCA, $\lambda = 0.7$ (RKPCA), and $\alpha = 5$ (SSC). Results in Figure \ref{fig:face_images} indicate that PMSDR enhances matrix recovery and outlier classification compared to the original methods.

In the second setup, we explore various subspace counts $L \in \{2, 3, 5, 8, 10, 12\}$, with subjects selected based on starting person IDs $sp \in \{1, 6, 11, 16, 21, 26\}$. Each $sp$ and $L$ configuration undergoes the same corruption process. The subspace dimension remains $8$, but other parameters are adjusted for different $L$ values. Table \ref{tab:face_CE} shows that outlier classification is highly effective with known ground truth. While \(\textbf{\textit{CE\textsubscript{recon}}}\) is lower due to subspace clustering errors, the $4$-stage pipeline shows strong potential in high-dimensional data analysis, even under unknown conditions.

In summary, the experiments demonstrate that our method is robust, adaptable, and effective for high-dimensional data, with strong potential for real-world applications where true subspace information may be limited.

\subsection{Experiment on Motion Segmentations}

\begin{algorithm}[h!]
    \caption{Normalization Trick for RKPCA}
    \label{alg:Normalization Trick}
    \textbf{Input}: Observed data matrix $\widetilde{\mathbf{G}}$; Parameters $\bm{\Theta}$. \\
    \textbf{Output}: Recovered data matrix $\widehat{\mathbf{G}}$
    \begin{algorithmic}[1]
        \STATE Compute column norms $V_{norms} = \left(\|\mathbf{\tilde{g}}_1\|, \ldots, \|\mathbf{\tilde{g}}_N\|\right)$
        \STATE Apply RKPCA to normalized data $\widehat{\mathbf{G}}_{n} = \mathrm{RKPCA}(\widetilde{\mathbf{G}}./V_{norms}; \bm{\Theta})$
        \STATE Rescale $\widehat{\mathbf{G}} = \widehat{\mathbf{G}}_{n} .* V_{norms}$
        \STATE \textbf{Return}: $\widehat{\mathbf{G}}$
    \end{algorithmic}
\end{algorithm}

We tested our algorithm on the Hopkins-155 database, which contains $117$ sequences with $2$ subspaces, $35$ with $3$ subspaces, and $1$ with $5$ subspaces. Each sequence is embedded in a $4$-dimensional linear subspace \cite{hop_dim4_1}\cite{hop_dim4_2}. We fixed the shuffled ratio at $0.4$, the outlier ratio at $0.4$, and the subspace dimension at $4$. The motion segmentation data is mapped from $3$D to $2$D coordinates to improve linear representation, with coordinates concatenated across all frames of a single motion. The resulting ambient space dimension post-preprocessing ranges from $40$ to $70$.

We applied our 4-stage pipeline Algorithm \ref{alg:algorithm_4_stage} (PMSDR) and compared it with RPCA, RKPCA, SSC, MRUC, and their respective PMSDR-augmented versions. We also explored the impact of using ground truth information for matrix recovery by applying matrix recovery to each group using known identities.

For RPCA, we set $\lambda = \frac{1}{\sqrt{\max(\#(\text{rows}), \#(\text{columns}))}}$ and $\mu = 10\lambda$. For RKPCA, the regularization parameter $\lambda$ was optimized in the range $0.05:0.05:0.6$ for each matrix. For SSC, $\alpha = 5$, treating each subspace as affine. For MRUC, we used the optimal initialization and set the initial rank to $2$ per matrix. We excluded the case of $5$ subspaces due to having only one sequence.

During PMSDR experiments, we calculated the recovery error (\textbf{RE}) as described earlier. In the second experiment, where global matrix recovery is considered, we calculated the recovery ratio for the entire matrix rather than restricting \textbf{RE} to outliers.

For RKPCA experiments, we applied a normalization trick (Algorithm \ref{alg:Normalization Trick}) to enhance performance, which was also used in the RKPCA method. During MRUC experiments, we did not independently permute each outlier; instead, outliers were divided into two parts, each permuted by the same matrix.

Tables \ref{tab:hopkins_PMSDR} and \ref{tab:hopkins_RPCA} present the experimental results. PMSDR demonstrated robust performance in both matrix recovery and outlier classification, even without ground truth information. Further improvements were observed when combining our pipeline with RPCA, SSC, and MRUC, although RKPCA did not show significant gains. With proper tuning, further improvements are expected, given that this experiment only used the best parameters for vanilla RKPCA.

\subsection{Impact of rank estimation}
 We use the synthetic data to show the impact of setting different ranks in Algorithm 1. Here, the ground-truth rank is 5, while all other parameters remain consistent with the experimental settings in the main paper. The reconstruction errors are reported in the following table. These results indicate that our pipeline can tolerate reasonably higher rank values $\hat{r}$.
\begin{table}[h!]
    \centering
    \begin{tabular}{c|c|c|c|c|c}
    \hline
\text{No. of subspaces} & $\hat{r}=3$ & $\hat{r}=4$ & $\hat{r}=\underline{5}$ & $\hat{r}=6$ & $\hat{r}=7$ \\ \hline
3& 0.8827 & 0.6578 & \textbf{0.1330} & 0.1405 & 0.2193 \\
10& 0.9164 & 0.6728 & \textbf{0.1372} & 0.1951 & 0.2226 \\ \hline
    \end{tabular}
    \caption{Impact of rank estimation on the recovery error}
    \label{tab:my_label}
\end{table}

\subsection{Impact of the number of subspaces and subspace dimension ($r$) on the recovery error}
Here we show the mean and standard deviation (10 trials) of recovery error with different $r$ and different number of subspaces.
\begin{table}[h!]
    \centering
\begin{tabular}{c|ccc}
\hline & subspace $=2$ & subspace $=5$ & subspace $=10$ \\
\hline$r=3$ & $0.0916 \pm 0.0726$ & $0.0847 \pm 0.0262$ & $0.1024 \pm 0.0323$ \\
$r=5$ & $0.1447 \pm 0.0484$ & $0.1772 \pm 0.0312$ & $0.1451 \pm 0.0259$ \\
$r=8$ & $0.2949 \pm 0.0416$ & $0.3019 \pm 0.0298$ & $0.3018 \pm 0.0164$ \\
$r=10$ & $0.4007 \pm 0.0300$ & $0.4096 \pm 0.0210$ & $0.4434 \pm 0.0229$ \\
$r=13$ & $0.6075 \pm 0.0426$ & $0.6307 \pm 0.0182$ & $0.6724 \pm 0.0144$ \\
$r=15$ & $0.7053 \pm 0.0207$ & $0.7685 \pm 0.0315$ & $0.8441 \pm 0.0160$ \\
\hline
\end{tabular}
\caption{Impact of the number of subspaces and subspace dimension ($r$) on the recovery error}
\end{table}

\section{Theory}

In this section, we analyze the theoretical foundations of our outlier classification algorithm, i.e., Algorithm \ref{alg:solve_partial_lsr}. Specifically, if the initial estimate of the subspaces is reasonably accurate and the number of outliers is relatively small, the proposed algorithm can reliably match each outlier to its corresponding subspace and recover the original data point $\mathbf{\hat{y}}$. As outlined in Algorithm \ref{alg:solve_partial_lsr}, the outlier classification procedure leverages least squares regression and iteratively eliminates the entries of the target vector and the rows of the basis with the largest residuals. Our goal is to demonstrate that this elimination process enhances the representation between the target vector $\bm{\nu}^{(i)}$ and the basis matrix $\mathbf{B}^{(i)}$, provided that $\bm{\nu}^{(i)}$'s underlying subspace is $\mathcal{S}_{\mathbf{B}^{(i)}}$.

To begin with, without loss of generality, consider a ground truth vector $\mathbf{y}$ where the entries' indices from $M_1+1$ to $M$ are permuted. That is:
\begin{align*}
    \mathbf{y} &= \begin{bmatrix}
    \mathbf{y}^{(1)} \\
    \mathbf{y}^{(2)}
    \end{bmatrix}, \quad
    \mathbf{\tilde{y}} = \begin{bmatrix}
    \mathbf{y}^{(1)} \\
    \mathbf{\tilde{y}}^{(2)}
    \end{bmatrix}
\end{align*}
where $\mathbf{\tilde{y}}$ is the permuted version of $\mathbf{y}$, and the unknown permutation matrix $\bm{\tilde{\pi}}$ is defined as follows:
\[
\bm{\tilde{\pi}} = \begin{bmatrix}
\mathbf{I}_{M_1} & 0 \\
0 & \bm{\pi} 
\end{bmatrix} \in \mathbb{R}^{M \times M}
\]
Here, $\mathbf{I}_{M_1}$ is the identity matrix of size $M_1 \times M_1$, and $\bm{\pi}$ is an unknown permutation matrix of size $M_2 \times M_2$ with $M_2 = M - M_1$. The permutation matrix $\bm{\tilde{\pi}}$ permutes the entries of $\mathbf{y}$ whose indices range from $M_1+1$ to $M$. Consequently, we have $\mathbf{\tilde{y}}^{(2)} = \bm{\pi} \mathbf{y}^{(2)}$ and $\mathbf{\tilde{y}} = \bm{\tilde{\pi}} \mathbf{y}$.

Nevertheless, we need to reformulate Theorem \ref{Theorem: Initial} step by step. Firstly, with Assumption \ref{Assumption:1}, we present the asymptotic distribution of order statistics.

\begin{theorem}(Theorem 10.3 in \cite{david2004order})
Let $X_1, \ldots, X_n$ be independently and identically distributed random variables with probability density function $f_X(x)$ and cumulative distribution function $F_X(x)$. Denote the maximum of these variables by $X_{(n)} \triangleq \max\{X_1, \ldots, X_n\}$. Then, as $n \rightarrow +\infty$ we have $X_{(n)}$ to be asymptotically normally distributed with mean $F_X^{-1}(1-\frac{1}{n})$ and variance $\frac{1}{n^2\left[f_X(F_X^{-1}(1-\frac{1}{n}))\right]^2}$. That is, 
\begin{align*}
X_{(n)} \sim \mathcal{AN}\left(F_X^{-1}(1-\frac{1}{n}), \frac{1}{n^2\left[f_X(F_X^{-1}(1-\frac{1}{n}))\right]^2}\right)
\end{align*}
\end{theorem} 

Let $\xi$ and $\eta$ be as defined in Assumption \ref{Assumption:1}. We can now reformulate equation (\ref{CDF: initial}) as follows:
\begin{align}
& F_\xi^{-1}(y) = \sigma_\xi \cdot F^{-1}(y),
& f_\xi(x) = \frac{1}{\sigma_\xi} f\left(\frac{x}{\sigma_\xi}\right) \\
& F_\eta^{-1}(y) = \sigma_\eta \cdot F^{-1}(y), 
& f_\eta(x) = \frac{1}{\sigma_\eta} f\left(\frac{x}{\sigma_\eta}\right).
\end{align}

Thus, we have
\begin{align}
\begin{split} \label{xi distribution}
    \xi_{(M_1)} & \sim \mathcal{AN}\left(\sigma_\xi F^{-1}\left(1-\frac{1}{M_1}\right), \frac{\sigma_\xi^2}{\left[M_1 f\left(F^{-1}\left(1-\frac{1}{M_1}\right)\right)\right]^2}\right) \\
    & \triangleq \mathcal{AN}\left(\sigma_\xi \rho(M_1), \sigma_\xi^2 \psi^2(M_1)\right)
\end{split}
\\
\begin{split}\label{eta distribution}
    \eta_{(M_2)} & \sim \mathcal{AN}\left(\sigma_\eta F^{-1}\left(1-\frac{1}{M_2}\right), \frac{\sigma_\eta^2}{\left[M_2f\left(F^{-1}\left(1-\frac{1}{M_2}\right)\right)\right]^2}\right) \\
    & \triangleq \mathcal{AN}\left(\sigma_\eta \rho(M_2), \sigma_\eta^2 \psi^2(M_2)\right)
\end{split}
\end{align}

Utilizing Gaussian statistics, we derive from (\ref{xi distribution}, \ref{eta distribution}) that
\begin{align}
\begin{split}
\Pr\left(\xi_{(M_1)} \leq \eta_{(M_2)}\right) & = \Phi \left(
\frac{\sigma_\eta\rho(M_2) - \sigma_\xi\rho(M_1)}{\sqrt{\sigma_\eta^2\psi^2(M_2) + \sigma_\xi^2\psi^2(M_1)}} 
\right)
\end{split}
\end{align}
Ultimately, we aim to find a significant lower bound for its argument. Specifically, it is straightforward to verify that $\frac{\partial \psi}{\partial \sigma_\xi} < 0$, while the sign of $\frac{\partial \psi}{\partial \sigma_\eta}$ is indeterminate.

Thus, we introduce the following theorem which provides a concentration upper bound for $\sigma_\xi^2$ and an estimate for $\mathbb{E}(\sigma_\eta^2)$.

\begin{theorem}
Let $\xi$ and $\eta$ be defined as above. We have the following approximated results:
\begin{align}
\begin{cases}
\Pr\left(
    \mathrm{Var}(\xi) < C\left(\frac{r(M-r)M_2}{M^2(M-1)(M+2)}\right)
    \right) 
> 
1-\delta 
\\
\mathbb{E}\left(\sigma_\eta^2\right) 
\approx
\frac{2M-M_2}{M^2}\left( \left(\frac{r}{M} - 1\right)^2 + \frac{2r(M - r)}{M^2(M + 2)} + \frac{(M_2 - 1)r(M - r)}{M(M - 1)(M + 2)} \right)
\end{cases}
\label{eq:final variance of target expression}
\end{align}
with $C \leq 2+6(1+\sqrt{2})\sqrt{\frac{M-r}{rM_2M}}$ and $\delta=0.0054$.
\end{theorem}

\subsection{Assumption and notations}

Let us denote
\[
\mathbf{U} = 
\begin{bmatrix}
    \mathbf{U}^{(1)} \\
    \mathbf{U}^{(2)}  
\end{bmatrix} 
= 
\begin{bmatrix}
    \mathbf{u}_1^\top \\
    \mathbf{u}_2^\top \\
    \vdots \\
    \mathbf{u}_M^\top
\end{bmatrix} 
\in \mathbb{R}^{M \times r}
\]
as the ground truth orthogonal basis from a subspace $\mathcal{S}$, where $\mathbf{U}^{(1)} \in \mathbb{R}^{M_1 \times r}$ and $\mathbf{U}^{(2)} \in \mathbb{R}^{M_2 \times r}$, with each $\mathbf{u}_i \in \mathbb{R}^r$ for $i = 1, 2, \ldots, M$. It is important to note that the choice of the orthogonal basis matrix $\mathbf{U}$ is inconsequential as long as the subspace $\mathcal{S}$ is fixed (even $\mathbf{U}$ can be non-orthogonal), which ensures that the projection matrix $\mathbf{U} \mathbf{U}^\top$ remains invariant.
The key assumptions for the theoretical guarantees include:
\begin{itemize}
    \item No noise is considered, which means $\mathbf{y} \in \mathcal{S}$ and $\mathbf{y} = \mathbf{UU^\top y}$.
    \item The permutation is partial, and permuted entries are randomly chosen, even though we regard them as ranging from $M_1+1$ to $M$. The permutation map $\varphi$ is defined such that:
    \begin{itemize}
        \item For $i \leq M_1$, $\varphi(i) = i$.
        \item For $M_1 +1 \leq i \leq M$, $\varphi(i) \neq i$, $\varphi^2(i) \neq i$ and $M_1 +1 \leq \varphi(i) \leq M$
    \end{itemize}
    Thus, the permutation matrix $\tilde{\pi}$ is represented as:
    \begin{align*}
        \bm{\tilde{\pi}} = 
        \begin{bmatrix}
            \mathbf{e}_1,  &
            \ldots, &
            \mathbf{e}_{M_1},  &
            \mathbf{e}_{\varphi(M_1+1)}, &
            \ldots,  &
            \mathbf{e}_{\varphi(M)}
        \end{bmatrix}^\top
    \end{align*}
    \item $\mathbf{y}$ is generated from a Gaussian distribution. Specifically, we take
    \begin{align*}
        \mathbf{y = U} \bm{\beta}, \text{where } \bm{\beta} \sim \mathcal{N}_r(0, \frac{1}{r} \mathbf{I}_r)
    \end{align*}
    Note that $\bm{\beta}$ is independent with the permutation matrix $\bm{\tilde{\pi}}$, subspace $\mathcal{S}$ and basis $\mathbf{U}$.
\end{itemize}

\subsection{The Target Expression}

This subsection delineates the analytical steps involved in solving the target expression, derived from the model assumptions and formulation. We begin by establishing the necessary mathematical framework for the analysis, followed by a detailed, step-by-step breakdown of the computational methods employed. The objective is to provide a clear and comprehensive derivation of the solution, ensuring that all relevant mathematical principles are correctly applied.

\begin{corollary}[Target Expression]\label{cor:target_expression}
Let \( Y(j) \triangleq \mathbf{\left(\tilde{y} - \hat{\tilde{y}}\right)}_j \), where \(\mathbf{\hat{\tilde{y}} = U U^{\top} \tilde{y}}\). Then, for each index \( j \in [M] \), we have:
\begin{equation}\label{eq:target_expression}
    Y(j) = \bm{\kappa}^\top \bm{\tau} \cdot \bm{\beta}
\end{equation}
Here, 
\begin{itemize}
\item \(\bm{\tau}_j 
    \triangleq 
    \mathbf{u}_j - \mathbf{u}_{\varphi(j)}\), 
    \quad $j = 1, \ldots, M$,
\item \(\bm{\tau} \triangleq 
    \begin{bmatrix}
    \bm{\tau}_1^\top \\ 
    \vdots \\ 
    \bm{\tau}_M^\top
    \end{bmatrix}
    =
    \left(
        \mathbf{I}_M - \bm{\tilde{\pi}}
    \right) \mathbf{U}
    = 
    \begin{bmatrix}
        \mathbf{0} \\
        \left(\mathbf{I}_{M_2} - \bm{\pi}\right) \mathbf{U}^{(2)}
    \end{bmatrix}\),
\item \(\bm{\kappa} 
    \triangleq \bm{\kappa}(j) 
    \triangleq \mathbf{U} \mathbf{u}_j - \mathbf{e}_j\), \quad $j = 1, \ldots, M$.
\end{itemize}
\end{corollary}

\begin{proof}\label{proof:target_expression}
    The proof begins by deriving expressions for \(Y\) based on the projection and transformation properties of the matrix \(\mathbf{U}\) and the vector \(\mathbf{y}\). We initiate with the noise vector \(\bm{\nu}\):
    \begin{equation}
        \begin{aligned}
            \mathbf{\tilde{y}} - \mathbf{\hat{\tilde{y}}} &= \bm{\tilde{\pi}} \mathbf{U U^\top y} - \mathbf{U U^\top} \bm{\tilde{\pi}} \mathbf{y} \\
            &= \bm{\tilde{\pi}} \left[ \mathbf{U U^\top} - \left(\bm{\tilde{\pi}}^\top \mathbf{U}\right)\left(\bm{\tilde{\pi}}^\top \mathbf{U}\right)^\top \right] \mathbf{y},
        \end{aligned}
        \label{eq:expression}
    \end{equation}
    where \(\bm{\pi}^\top \mathbf{U}\) can intuitively be considered as a basis from another subspace, \(\widetilde{\mathcal{S}}\). Consequently, we find:
    \begin{align*}
        \bm{\tilde{\pi}^\top}(\mathbf{\tilde{y}} - \mathbf{\hat{\tilde{y}}} ) &= \left[ \mathbf{U U}^\top - \left(\bm{\tilde{\pi}}^\top \mathbf{U}\right)\left(\bm{\tilde{\pi}}^\top \mathbf{U}\right)^\top \right] \mathbf{y}
    \end{align*}
    which implies:
    \begin{align*}
        Y(\varphi^{-1}(j)) 
        &= 
        \mathbf{u}_j^\top \mathbf{U}^\top \mathbf{y} - \mathbf{u}_{\varphi^{-1}(j)}^\top (\mathbf{U}^\top \bm{\tilde{\pi}}) \mathbf{y} \\
        &= \mathbf{u}_j^\top \bm{\beta} - \mathbf{u}_{\varphi^{-1}(j)}^\top 
        \left(
        {\mathbf{U}^{(1)}}^\top \mathbf{U}^{(1)} + {\mathbf{U}^{(2)}}^\top \mathbf{U}^{(2)} + {\mathbf{U}^{(2)}}^\top \left(\bm{\pi} - \mathbf{I}_{M_2}\right) \mathbf{U}^{(2)}
        \right) \bm{\beta} \\
        &= \left(\mathbf{u}_j - \mathbf{u}_{\varphi^{-1}(j)}\right)^\top \bm{\beta} \quad + \mathbf{u}_{\varphi^{-1}(j)}^\top {\mathbf{U}^{(2)}}^\top \left(\mathbf{I}_{M_2} - \bm{\pi}\right) \mathbf{U}^{(2)} \bm{\beta}
    \end{align*}
    Thus the derived expression is thus proven.
\end{proof}

If we consider the bases $\mathbf{U}$ and $\bm{\tilde{\pi}}$ to be fixed and assume that $r$ is constant, we can readily deduce that $Y(j) \sim \mathcal{N}\left(0, \frac{1}{r}\bm{\kappa}^\top \bm{\tau\tau}^\top \bm{\kappa}\right)$. Consequently, the variance is immediately obtained as follows:
\begin{align}
\begin{split}
\sigma_\xi^2
&= \frac{1}{r} {\bm{\kappa}(1)}^\top \bm{\tau\tau}^\top \bm{\kappa}(1),
\end{split}
\\
\begin{split}
\sigma_\eta^2
&= \frac{1}{r} {\bm{\kappa}(M)}^\top \bm{\tau\tau}^\top \bm{\kappa}(M).
\end{split}
\end{align}

However, to appropriately address the scenario where $\mathbf{U}$ is random, we must first establish some prerequisite knowledge.

\subsection{Prerequisite Knowledge}

A substantial body of literature has explored the relationship between random Gaussian matrices and orthogonal matrices. Much of this research focuses on the various conditions under which an orthogonal matrix asymptotically approaches a Gaussian matrix.

For example, \cite{davidson2001local}, \cite{jiang2006many}, and \cite{jiang2019plot} indicate that, in the limit as matrix size increases and under certain conditions, an upper-left submatrix \( \mathbf{Z} \in \mathbb{R}^{p \times q} \) of a Haar-invariant orthogonal matrix \( \mathbf{\Gamma} \in \mathbb{R}^{M \times M} \) behaves similarly to a matrix of independent normal distributions \(\mathcal{N}(0, \frac{1}{M}\mathbf{I}_p \otimes \mathbf{I}_q) \). Over time, studies have progressively relaxed these conditions, but the upper bound on the submatrix size relative to the dimensions of the ambient space must approach zero. This condition cannot be satisfied in our context, as we assume the ratio \(\frac{M_2}{M}\) to be fixed and not approaching zero. Therefore, we require an explicit distribution to study the properties of a column-wise orthogonal matrix \(\mathbf{U}\).

To begin with, we directly cite some existing results on the matrix variate beta distribution, primarily from \cite{Matrix_variate_distributions} and \cite{Random_Orthogonal_Matrices}.

\begin{definition}[Matrix Variate Beta Distribution] \label{Def: Matrix Variate Beta Distribution}\cite{Matrix_variate_distributions}
A \( p \times p \) random symmetric positive definite matrix \( \mathbf{H} \) is said to have a matrix variate beta type I distribution with parameters \( (a, b) \), denoted as \( \mathbf{H} \sim \mathcal{B}_p(a, b) \), if its density function is given by
\begin{equation} \label{eq:basic matrix variate beta distribution}
    \left\{ \bm{\beta}_p(a, b) \right\}^{-1} \det(\mathbf{H})^{a - \frac{1}{2}(p + 1)} \det(\mathbf{I}_p - \mathbf{U})^{b - \frac{1}{2}(p + 1)}, \\ \mathbf{0} < \mathbf{U} < \mathbf{I}_p,
\end{equation}
where \( a > \frac{1}{2}(p - 1) \), \( b > \frac{1}{2}(p - 1) \), and \( \beta_p(a, b) \) is the multivariate beta function, which normalizes the density function.
\end{definition}

\begin{theorem}[Transformation of Matrix Variate Beta Distribution] \cite{Matrix_variate_distributions}
    Let \( \mathbf{H} \sim \mathcal{B}_p(a, b) \). Then for given \( p \times p \) symmetric matrices \( \bm{\Psi} (\geq 0) \) and \( \bm{\Omega (> \Psi)} \), the random matrix \( \mathbf{X} (p \times p) \) defined by
    \[
    \mathbf{X} = (\bm{\Omega - \Psi})^{\frac{1}{2}} \mathbf{H} (\bm{\Omega - \Psi})^{\frac{1}{2}} + \bm{\Psi}
    \]
    has the density function
    \begin{equation}\label{expr:Generalized Matrix Variate Beta Distribution}
        \frac{\det(\mathbf{X} - \bm{\Psi})^{a - \frac{1}{2}(p + 1)} \det(\bm{\Omega} - \mathbf{X})^{b - \frac{1}{2}(p + 1)}}{\bm{\beta}_p(a, b) \det(\bm{\Omega - \Psi})^{(a + b - \frac{1}{2}(p + 1))}}, \quad \Psi < X < \Omega.   
    \end{equation}
\end{theorem}

\begin{definition}[Generalized Matrix Variate Beta Distribution] \cite{Matrix_variate_distributions}
A \( p \times p \) random symmetric positive definite matrix \( \mathbf{X} \) is said to have a generalized matrix variate beta distribution with parameters \( a, b; \bm{\Omega, \Psi} \) denoted by \( \mathbf{X} \sim \mathcal{GB}_p(a, b; \bm{\Omega, \Psi}) \) if its density function is given by (\ref{expr:Generalized Matrix Variate Beta Distribution}).
\end{definition}

\begin{lemma} \cite{Matrix_variate_distributions} \label{lemma: basic2generalized beta distribution}
    Let \( \mathbf{H} \sim \mathcal{B}_p\left(\frac{r}{2}, \frac{M-r}{2}\right) \) and \( \mathbf{A} \in \mathbb{R}^{q \times p} \) be a constant nonsingular matrix. Then \( \mathbf{AHA'} \sim \mathcal{GB}_q\left(\frac{r}{2}, \frac{M-r}{2}; \mathbf{AA'}, \mathbf{0}\right) \).
\end{lemma}

Thus, we can regard a basic matrix variate beta distribution  \(\mathcal{B}_p(a, b)\) as \(\mathcal{GB}_p\left(a, b; \mathbf{I}_p, \mathbf{0}\right)\). Then the following properties will be useful.

\begin{theorem} \label{Covariance of Generalized Beta Distribution}
\cite{Matrix_variate_distributions}
Let \( \mathbf{H} \sim \mathcal{GB}_p \left( \frac{r}{2}, \frac{M-r}{2}; \bm{\Omega}, 0 \right) \), then
\begin{itemize}
    \item[(i)] \( \mathbb{E}(H_{ij}) = \frac{r}{M} \omega_{ij} \)
    \item[(ii)]
    \(
      \mathbb{E}(H_{ij} H_{k\ell}) = \frac{r}{M(M-1)(M+2)} \left[ r\{(M+1) - 2 \} \omega_{ij} \omega_{k\ell} \right. \\
       \quad \quad \quad \quad \quad 
       \left. +  (M-r) (\omega_{j\ell} \omega_{ik} + \omega_{i\ell} \omega_{kj}) \right] 
    \)
    
\end{itemize}
where \( \mathbf{H} = (H_{ij}) \), and \( \bm{\Omega} = (\omega_{ij}) \).
\end{theorem}

From the above result, we can easily obtain the following corollary for \(\mathbf{H} \sim \mathcal{B}_p(\frac{r}{2}, \frac{M-r}{2})\):

\begin{corollary} \label{corollary:statistic_properties_of_Matrix_Variate_Beta_Distribution}
For a random matrix \( \mathbf{H} \sim \mathcal{B}_p(\frac{r}{2}, \frac{M-r}{2}) \), we have:
\begin{align}
\mathbb{E}[H_{ij}] &= 
\begin{cases}
    \frac{r}{M}, & \text{if } i = j; \\
    0, & \text{if } i \neq j
\end{cases} \label{eq:expectation_uij} \\
\mathrm{Var}(H_{ij}) &= 
\begin{cases}
    \frac{2r(M-r)}{M^2(M+2)}, & \text{if } i = j; \\
    \frac{r(M-r)}{M(M-1)(M+2)}, & \text{if } i \neq j
\end{cases} \label{eq:variance_hij} \\
\mathrm{Cov}(H_{ij}, H_{k\ell}) &= 
\begin{cases}
    \frac{-2r(M-r)}{M^2(M-1)(M+2)}, \\
    \quad \quad \text{if } i = j, k = \ell, i \neq k; \vspace{+2.5mm}\\
    \frac{r(M-r)}{M(M-1)(M+2)}, \\
    \quad \quad \text{if } i \neq j, k \neq \ell, \text{ and } i = \ell, j = k; \vspace{+2.5mm}\\
    0, \quad \text{otherwise, except for } i = k, j = \ell.
\end{cases} \label{eq:covariance_hij_hkl}
\end{align}
\end{corollary}

\begin{proof}
    These results follow directly from Theorem \ref{Covariance of Generalized Beta Distribution}.
\end{proof}

\begin{corollary}\label{corollary:third-order moments}
(1) Each diagonal element \( H_{ii} \) follows a Beta distribution \( \bm{\beta}\left(\frac{r}{2}, \frac{M-r}{2}\right) \). Therefore, we have:
\begin{align}
\begin{cases}
    \mathbb{E}(H_{ii}^3) &= \frac{r(r+2)(r+4)}{M(M+2)(M+4)}, \\
    \mathbb{E}[(H_{ii} - \mathbb{E}(H_{ii}))^3] &= \frac{8(M - r)(M - 2r)r}{M^3 (M + 2) (M + 4)}.
\end{cases}
\end{align}

(2) Each off-diagonal element \( H_{ij} \) follows a symmetric distribution with the mean and variance given in Corollary \ref{corollary:statistic_properties_of_Matrix_Variate_Beta_Distribution}. Additionally, under the conditions \( i \neq j \), \( s \neq t \), and \( k \neq \ell \), the following holds:
\begin{align}
\begin{cases}
    \mathbb{E}(H_{ij}^3) &= 0, \\
    \mathbb{E}(H_{ij}H_{st}^2) &= 0, \\
    \mathbb{E}(H_{ij}H_{st}H_{k\ell}) &= 0, \quad \text{if } j \neq s, t \neq k, \text{ or } \ell \neq i, \\
    \mathbb{E}(H_{ii}H_{st}H_{k\ell}) &= 0, \quad \text{if } s \neq k \text{ or } t \neq \ell.
\end{cases}
\end{align}

(3) Under the condition \( s \neq t \), the following inequality holds:
\begin{align}
\begin{cases}
    0 \leq \mathbb{E}(H_{ii}H_{st}^2) \leq \sqrt{\mathbb{E}(H_{ii}^2)\mathbb{E}(H_{st}^4)}, \\
    0 \leq \mathbb{E}(H_{ii}H_{ss}^2) \leq \sqrt{\mathbb{E}(H_{ii}^2)\mathbb{E}(H_{ss}^4)}.
\end{cases}
\end{align}
\end{corollary}


\begin{proof}
(1) By applying Lemma \ref{lemma: basic2generalized beta distribution} with \( \mathbf{A} = \mathbf{e}_i \), the result for the diagonal elements follows.

(2) For the off-diagonal elements, using the transformation \( \mathbf{A} = \left[\mathbf{e}_1, \ldots, -\mathbf{e}_i, \ldots, \mathbf{e}_p\right] \), we see that the distribution of \( \mathbf{H} \) is unchanged except for flipping the sign of the \( H_{ij} \)-th entry. This symmetry leads to \( \mathbb{E}(H_{ij}^3) = 0 \). Similarly, the symmetry of terms \( H_{ij}H_{st}^2 \) and \( H_{ij}H_{st}H_{k\ell} \) ensures that these expectations are zero unless specific conditions are met, in which the sign-changing technique cannot be applied.

(3) The result follows immediately from the application of Hölder's inequality.


\end{proof}

\begin{theorem} \cite{Matrix_variate_distributions}
\label{theorem: E(det(H)^t)}
Let $\mathbf{H} \sim \mathcal{B}_p\left(\frac{r}{2}, \frac{M-r}{2}\right)$. Then,
\begin{align}
\mathbb{E}[\det(\mathbf{H})^k] = \frac{\Gamma_p\left(\frac{r}{2}+k\right)\Gamma_p\left(\frac{M}{2}\right)}{\Gamma_p\left(\frac{M}{2}+k\right)\Gamma_p\left(\frac{r}{2}\right)}, \quad k>0,
\end{align}
where
\[
\Gamma_p(x) = \pi^{\frac{p(p-1)}{4}} \prod_{i=1}^p \Gamma\left(x - \frac{i-1}{2}\right).
\]
\end{theorem}

Building on the orthogonal invariance property of the matrix-variate Beta distribution (Lemma \ref{lemma: basic2generalized beta distribution}), we can obtain the following property of $\mathbf{H}$:
\begin{lemma}\cite{bedoya2007moments}
    For any $t \in \mathbb{R}$, the expectation satisfies $\mathbb{E}(\mathbf{H}^t) = \tilde{c}_t(p, r, M - r) \cdot \mathbf{I}_p$, indicating that it is a scalar multiple of the identity matrix.
\end{lemma}

Furthermore, from the results in \cite{Matrix_variate_distributions}, we can derive the following specific cases:
\begin{align}
    \tilde{c}_1(p, r, M - r) &= \frac{r}{M}; \\
    \tilde{c}_2(p, r, M - r) &= \frac{r\left[r(M+1)-2+(M-r)(p+1)\right]}{M(M-1)(M+2)}.
\end{align}
To the best of our knowledge, no explicit expressions have been established for $\tilde{c}_t(p, r, M - r)$ when $t \geq 3$ in the context of the \textbf{real-valued} matrix-variate Beta distribution, as defined in Definition \ref{Def: Matrix Variate Beta Distribution}. Existing literature, such as \cite{bedoya2007moments}, only addresses the complex-valued case for $t=3$, leaving the real-valued case unresolved and open for future exploration. Nevertheless, an approximation for the case $t=3$ can be formulated as follows.

Given that
\begin{align*}
    (H)_{ii} &= H_{ii}, \\
    (H^2)_{ii} &= \sum_{j=1}^p H_{ij} H_{ji}, \\
    (H^3)_{ii} &= \sum_{j=1}^p \sum_{k=1}^p H_{ij} H_{jk} H_{ki},
\end{align*}
it is evident that
\begin{align*}
    \tilde{c}_3(M, r, M - r) &= \frac{r}{M}, \quad
    \tilde{c}_3(p, M, 0) = 1.
\end{align*}
This allows us to recognize the patterns for higher-order terms, and we can express:
\begin{align}\label{order-3}
    \tilde{c}_3(p, r, M - r) \sim O\left(\left(\frac{p}{M}\right)^2\right) \approx \left(\frac{p}{M}\right)^2\cdot\frac{r}{M}.
\end{align}

Moreover, to generalize $\tilde{c}_3(p, r, M - r) \approx \left(\frac{p}{M}\right)^2 \tilde{c}_3(M, r, M - r)$, we propose the following approximation:
\begin{align}
    \tilde{c}_3(q, r, M - r) \approx \left(\frac{q}{p}\right)^2 \tilde{c}_3(p, r, M - r),
\end{align}
which implies the approximation
\begin{align}
    \mathbb{E}\left( \sum_{j=1}^q \sum_{k=1}^q H_{ij} H_{jk} H_{ki} \right)
    \approx \left(\frac{q}{p}\right)^2
    \mathbb{E}\left( \sum_{j=1}^p \sum_{k=1}^p H_{ij} H_{jk} H_{ki} \right).
\end{align}

\newcommand{\randommatrix}[1]{\mathcal{M}_{#1}}
Next, we establish the relationship between the matrix variate beta distribution and the random orthogonal matrix. When considering the randomness of a random column-wise orthogonal matrix \( \mathbf{U} \in \randommatrix{M \times r}\), we first define \(\mathbf{X} \sim \mathcal{N}_M(\mathbf{0}, \frac{1}{M}\mathbf{I}_M \otimes \mathbf{I}_r)\). Since \(\mathbf{X^\top X}\) is nonsingular with probability $1$, we can define, without loss of generality,
\begin{equation} \label{def: column wise orthogonal matrix}
    \mathbf{U = X(X^\top X)}^{-\frac{1}{2}}.
\end{equation}

It is straightforward to verify that \(\mathbf{U}  \) is a column-wise orthogonal matrix with columns having \(l_2\)-norm equal to $1$, which satisfies our requirements. We omit the detailed discussion of the Haar-invariant property of \(\mathbf{U}  \) as it is beyond the scope of this paper. Nevertheless, the following proposition establishes a significant relationship:

\begin{proposition} \cite{Random_Orthogonal_Matrices} \label{prop: matrix product to beta}
    Let $\Delta \in  \randommatrix{p \times q}$ be the upper-left submatrix of a random orthogonal matrix $\Gamma \in  \randommatrix{M \times M}$. When \( q \leq p \) and \( p + q \leq M \), the random matrix \( \Delta' \Delta \) has a \( \mathcal{B}_q(\frac{p}{2}, \frac{M-p}{2}) \) distribution.
\end{proposition}

From Proposition \ref{prop: matrix product to beta}, we can derive some properties of the expression \( \mathbf{U}^{(2)}\mathbf{u}_j \) in equation (\ref{eq:target_expression}). However, since \( \mathbf{U}^{(2)} {\mathbf{U}^{(2)}}^\top\) does not meet the condition of Proposition \ref{prop: matrix product to beta} due to our assumption that the subspace rank \( r \ll M_2 \), further properties should be studied. 

Now, let's present the property we studied about \(\mathbf{U}^{(2)} {\mathbf{U}^{(2)}}^\top\):

\begin{lemma} \label{lem: covariance properties}
    Let $\Delta \in  \randommatrix{p \times q}$ be the submatrix consisting of randomly chosen \(p\) rows and \(q\) columns of a random orthogonal matrix $\mathbf{\Gamma} \in  \randommatrix{M \times M}$. When \( q \leq p \) and \( p + q \leq M \), the entries of the random matrix $\Delta \Delta^\top$ have the same statistical moments as \(\mathcal{GB}_p\left(\frac{1}{2}q, \frac{1}{2}(M-q); \mathbf{I}_r, \mathbf{0}\right)\), specifically Theorem \ref{Covariance of Generalized Beta Distribution}, even though $\Delta \Delta^\top$ does not follow a matrix variate beta distribution.
\end{lemma}

To demonstrate Lemma \ref{lem: covariance properties}, we need another proposition which is adapted from Proposition 7.1 in \cite{Random_Orthogonal_Matrices}, presented as follows.

\begin{proposition} \label{Prop: invariate with orthogonal matrix}
    For a random column-wise orthogonal matrix $\mathbf{\Gamma}_1$ defined by (\ref{def: column wise orthogonal matrix}) and fixed orthogonal transformations $f, g$, $f\mathbf{\Gamma}_1g$ and $\mathbf{\Gamma}_1$ share the same distribution.
\end{proposition}

Thus, we can move any $p$ rows and $q$ columns to the upper-left block with some fixed permutation matrix and obtain the desired results. Specifically, we can handle any $r$ rows of $\mathbf{U}^{(2)}$, denoted as $\mathbf{Ur} \in  \randommatrix{r \times r}$, and apply Proposition \ref{prop: matrix product to beta} and Theorem \ref{Covariance of Generalized Beta Distribution} to $\mathbf{Ur}$. Then Lemma \ref{lem: covariance properties} is proved.


\begin{figure*}[t]
    \centering
    \includegraphics[width=1.05\textwidth,trim={0 40 0 0},clip]{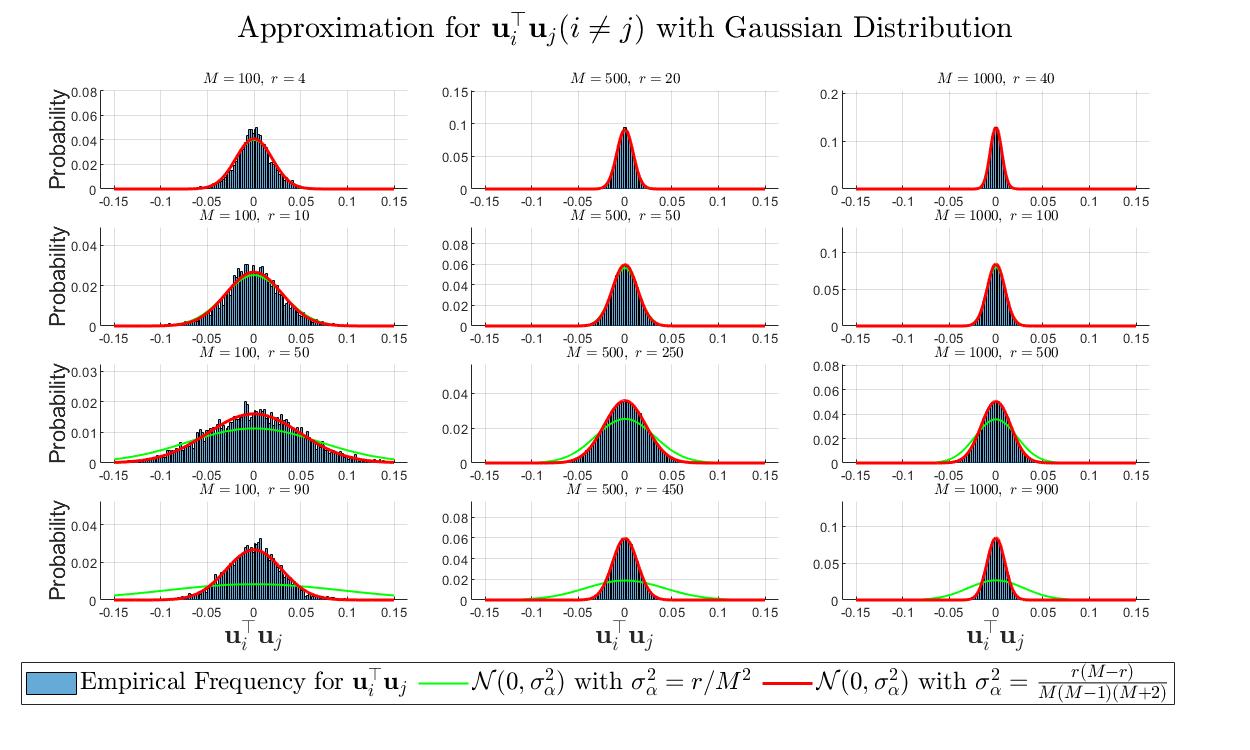}
    \caption{Approximation performance of the Gaussian distribution for $\bm{\alpha} = {\mathbf{U}^{(2)}}\mathbf{u}_j $ under different combinations of $M$ and $r$. The blue histogram represents the empirical frequency of $\mathbf{u}_i^\top \mathbf{u}_j$, which are extracted from all entries in the upper triangular part of one same matrix $\mathbf{U}\mathbf{U}^{\top}$. The green and red curves represent normal distribution fits with $\sigma_{\alpha}^2 = \frac{r}{M^2}$ and $\sigma_{\alpha}^2 = \frac{r(M-r)}{M(M-1)(M+2)}$, respectively. For small $r$, both approximations closely match the empirical distribution. As $r$ increases, the fit with $\sigma_{\alpha}^2 = \frac{r(M-r)}{M(M-1)(M+2)}$ becomes more accurate, highlighting the need for precise $\sigma_{\alpha}^2$ calculation and demonstrating the effectiveness of the Gaussian approximation.}
    \label{fig:orth_v7}
\end{figure*}

\subsection{Approximation of Target Expression}

Let us recall the target expression (\ref{eq:target_expression}):

\begin{equation*}
    Y(j) = \bm{\kappa}^\top\bm{\tau}\cdot\bm{\beta} \sim \mathcal{N}\left(0, \frac{1}{r}\bm{\kappa}^\top \bm{\tau\tau}^\top \bm{\kappa}\right) \quad \text{for } j = 1, \ldots, M
\end{equation*}

In the following discussion, we will not focus on precise derivations but will make reasonable assumptions and provide an approximate analysis, supported by synthetic experiments.

\subsubsection{Cases for $\sigma_\xi^2$:}
We denote $\mathbf{H}=[H_{ij}]_{M\times M}=\mathbf{U} \mathbf{U}^\top$. Let $\bm{\alpha} \triangleq \mathbf{U}^{(2)}\mathbf{u}_1$. By invoking Corollary \ref{corollary:statistic_properties_of_Matrix_Variate_Beta_Distribution} and Lemma \ref{lem: covariance properties}, the covariance of $\bm{\alpha}$ can be readily computed as $\mathrm{Cov}(\bm{\alpha}) = \sigma_{\alpha}^2 \cdot \mathbf{I}_{M_2}$, where
\begin{equation}\label{var: of U*uj}
    \sigma_{\alpha}^2 = \mathrm{Var}\left(\mathbf{u}_i^\top \mathbf{u}_1\right) = \frac{r(M-r)}{M(M-1)(M+2)}, \quad i \geq M_1+1.
\end{equation}
Thus, we treat $\bm{\alpha}$ as a random vector following a Gaussian distribution $\mathcal{N}(\mathbf{0}, \sigma_{\alpha}^2 \mathbf{I}_{M_2})$. In a similar manner, we assume that all off-diagonal elements of $\mathbf{H}$ are i.i.d. and distributed as $\mathcal{N}(0, \sigma_{\alpha}^2)$. As shown in Figure \ref{fig:orth_v7}, the Gaussian approximation with an accurately estimated variance yields strong performance. Consequently, we can confidently approximate $\mathbb{E}(H_{ij}^4)$ as $3\mathbb{E}(H_{ij}^2)^2 = \frac{3r^2(M-r)^2}{M^2(M-1)^2(M+2)^2}$.

Since that:
\begin{align*}
    \bm{\tau}^\top\bm{\kappa}(1) 
    &= 
    {\mathbf{U}^{(2)}}^\top
    \left(\mathbf{I}_{M_2}-\bm{\pi}\right) \mathbf{U}^{(2)}\mathbf{u}_1 \\
    &=
    \sum_{i=M_1+1}^{M}\left(\mathbf{u}_i{\mathbf{u}_i}^\top - \mathbf{u}_i\mathbf{u}_{\varphi(i)}^\top\right) \mathbf{u}_1 \\
    &=
    \sum_{i=M_1+1}^{M}\mathbf{u}_i \left(H_{i1} -  H_{\varphi(i) 1}\right)
\end{align*}

Thus, the expression for $\sigma_\xi^2$ becomes:

\begin{align}\label{var(alpha nu) when j <= n1; Half}
\begin{split}
\sigma_\xi^2 
&= \frac{1}{r} \cdot \bm{\kappa}^\top(1) \bm{\tau\tau}^\top \bm{\kappa}(1) \\
&= \frac{1}{r}\sum_{p=M_1+1}^{M}\sum_{i=M_1+1}^{M} H_{pi}
\left(H_{p1} -  H_{\varphi(p) 1}\right)\left(H_{i1} -  H_{\varphi(i) 1}\right)
\end{split}
\end{align}

By leveraging Corollary \ref{corollary:third-order moments} and relation (\ref{order-3}), we deduce that:

\begin{align}
\begin{split}
\mathbb{E}(\sigma_\xi^2) 
&=
\frac{2}{r} \sum_{i=M_1+1}^M 
\mathbb{E}\left[H_{ii}H_{i1}^2\right] + \frac{1}{r}{\sum\sum}_{p\neq i}\mathbb{E}\left[
H_{1p}H_{pi}H_{i1}
\right]
\\
&=
\frac{1}{r} \sum_{i=M_1+1}^M 
\mathbb{E}
\left[
    \mu_{ii}\mathbb{E}(H_{i1}^2) + \mathbb{E}((H_{ii}- \mu_{ii})H_{i1}^2)
\right] + 
\frac{1}{r}\sum^M_{p=M_1+1}\sum^M_{i=M_1+1}\mathbb{E}\left[
H_{1p}H_{pi}H_{i1}
\right]
\\
&\approx
\frac{1}{r} \sum_{i=M_1+1}^M 
\mathbb{E}
\left[
    \mu_{ii}\mathbb{E}(H_{i1}^2) + \mathbb{E}((H_{ii}- \mu_{ii})H_{i1}^2)
\right] + 
\frac{1}{r}\cdot\left(\frac{M_2}{M}\right)^2\tilde{\mathrm{c}}_3(M, r, M-r) \\
&\leq \frac{M_2}{r}
\cdot
\left[
\frac{r^2(M-r)}{M^2(M-1)(M+2)}+\sqrt{\mathrm{Var}(H_{ii})\mathbb{E}(H_{i1}^4)}
\right] \\
&=
\frac{M_2r(M-r)}{M^2(M-1)(M+2)} + \frac{1}{r}\frac{M_2^2}{M^2}\frac{r}{M} +\frac{M_2}{r}\cdot \sqrt{\frac{2r(M-r)}{M^2(M+2)}\frac{3r^2(M-r)^2}{M^2(M-1)^2(M+2)^2}} \\
&= \frac{M_2r(M - r)}{M^2(M - 1)(M + 2)} + \frac{\sqrt{6}M_2r^{1/2}(M - r)^{3/2}}{M^2(M - 1)(M + 2)^{3/2}}+\frac{M_2^2}{M^3}
\end{split}
\end{align}
where $\mu_{ii} \triangleq \mathbb{E}(H_{ii}) =\frac{r}{M}$.

Nevertheless, our primary objective is to establish a concentration bound. To achieve this, we utilize the Gaussian distribution to derive an approximate upper bound. Let's assume $\bm{\alpha}$ independent of $\bm{\tau}$. Then we introduce a useful lemma:
\begin{lemma} \label{lem: var of product}
    Let \( X \) and \( Y \) be zero-mean random variables such that \( \mathrm{Cov}(X) = \sigma^2 \mathbf{I}_M \) and \( \mathrm{Cov}(Y) = \bm{\Sigma} \), where \( X \) and \( Y \) are independent. Then,
    \begin{align*}
        \mathrm{Var}(X^\top Y) = \sigma^2 \mathrm{tr}(\bm{\Sigma}).
    \end{align*}
\end{lemma}

\begin{proof}
    We can express the variance of the product as:
    \begin{align*}
        \mathrm{Var}(X^\top Y) &= \sum_{i = 1}^M \mathrm{Var}(X_i Y_i) + \sum_{i \neq j} \mathrm{Cov}(X_i Y_i, X_j Y_j) \\
        &= \sum_{i = 1}^M \mathrm{Var}(X_i) \mathrm{Var}(Y_i) + \sum_{i \neq j} (\mathbb{E}[X_i X_j Y_i Y_j] - \mathbb{E}[X_i Y_i] \mathbb{E}[X_j Y_j]) \\
        &= \sum_{i = 1}^M \sigma^2 \Sigma_{ii} + 0 \\
        &= \sigma^2 \mathrm{tr}(\bm{\Sigma}).
    \end{align*}
    The second equality holds because:
    \begin{align*}
        \mathrm{Var}(X_i Y_i) &= \mathrm{Var}(X_i) \mathrm{Var}(Y_i) + \mathrm{Var}(X_i)(\mathbb{E}[Y_i])^2  + \mathrm{Var}(Y_i)(\mathbb{E}[X_i])^2,
    \end{align*}
    as long as \( X_i \) and \( Y_i \) are independent.
\end{proof}

Thus, assuming that \( \bm{\alpha} \) is independent of \( \bm{\tau} \), we can compute \( \sigma_\xi^2 \) as follows:
\begin{align}\label{var(alpha nu) when j <= n1; Half}
\begin{split}
\sigma_\xi^2
&= \frac{r(M-r)}{M(M-1)(M+2)} \cdot \mathrm{tr}(\bm{\Sigma}) \\
&= \frac{r(M-r)}{M(M-1)(M+2)} \cdot \frac{1}{r}\sum_{i = M_1+1}^{M}\bm{\tau}_i^\top \bm{\tau}_i \\
&= \frac{2(M-r)}{M(M-1)(M+2)} \cdot \sum_{i = M_1+1}^{M} \left(H_{ii} - H_{i\varphi(i)}\right).
\end{split}
\end{align}

Consequently, we approximate:
\begin{align*}
    -\sum_{i = M_1+1}^{M}H_{i\varphi(i)} \sim \mathcal{N}\left(0, \frac{M_2 r (M-r)}{M(M-1)(M+2)}\right),
\end{align*}
and apply the $3\sigma$ principle.

Moreover, from Corollary \ref{corollary:statistic_properties_of_Matrix_Variate_Beta_Distribution}, we can readily compute:
\begin{align*}
\mathbb{E} \left( \sum_{i = M_1+1}^{M} H_{ii} \right) &= \frac{rM_2}{M}, \\ 
\mathrm{Var} \left( \sum_{i = M_1+1}^{M} H_{ii} \right) &= \frac{2M_2(M-M_2)r(M-r)}{M^2(M+2)(M-1)}.
\end{align*}

Thus, with the mean and variance at hand, we can accurately simulate both the Beta and Gaussian distributions, as shown in Figure \ref{fig:orth_v11}.

\begin{figure*}[htbp]
    \centering
    \includegraphics[width=1.04\textwidth,trim={0 40 0 0},clip]{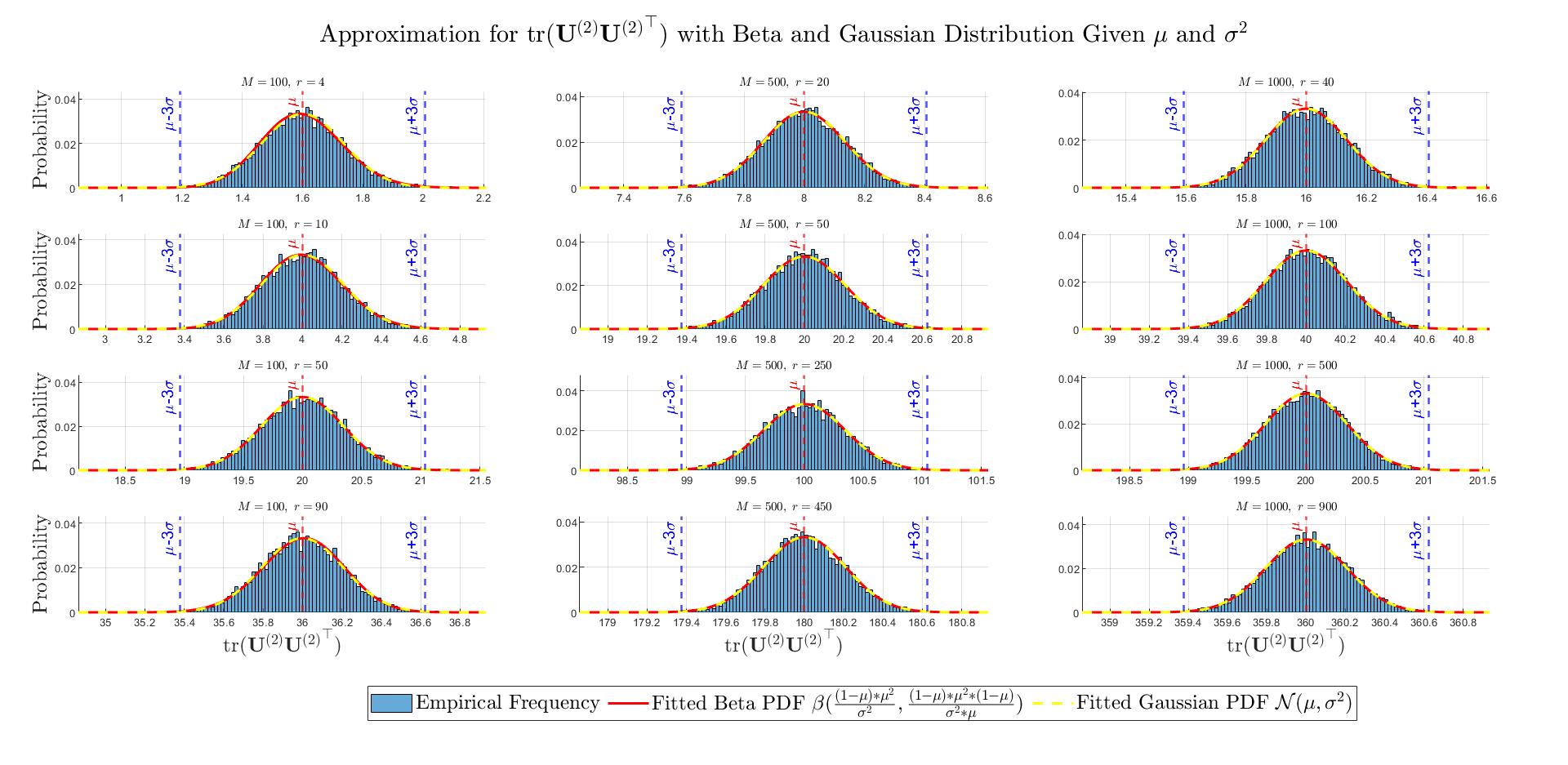}
    \caption{Approximation for $\mathrm{tr}(\mathbf{U}^{(2)}\mathbf{U}^{(2)\top})$ using Beta and Gaussian distributions, given the mean $\mu = \frac{rM_2}{M}$ and variance $\sigma^2=\frac{2M_2(M-M_2)r(M-r)}{M^2(M+2)(M-1)}$. The shuffled ratio \( \frac{M_2}{M} \) is fixed at 0.4 in all cases. These plots show that both Beta and Gaussian distributions provide highly accurate approximations when the mean and variance are properly estimated. The application of the $3\sigma$ principle further validates the accuracy of these approximations, effectively capturing the empirical frequency distribution. This demonstrates the robustness of these statistical models for approximation purposes across a range of parameter settings.}
    \label{fig:orth_v11}
\end{figure*}

Finally, applying the $3\sigma$ rule to the sum \( \sum_{i = M_1+1}^{M} H_{ii} \), we can derive an upper bound for \( \sigma_\xi^2 \) with high probability:
\begin{align} \label{Var(z_j): j<=n1}
    \Pr \left[\sigma_\xi^2 \leq \mathrm{g}(M, M_2, r)\right] \geq 1-\delta, 
\end{align}
where $\delta = 1-0.9973^2 = 0.0054$ is small, and
\begin{align} 
    \begin{split}
        \mathrm{g}(M, M_2, r) &= \frac{2(M-r)}{M(M-1)(M+2)} \left[ \frac{rM_2}{M}  + 3\sqrt{\frac{2M_2(1-M_2/M) r(M-r)}{M(M-1)(M+2)}}   + 3\sqrt{\frac{r(M-r)M_2}{M(M-1)(M+2)}} \right] \\
        & = C \cdot \frac{r(M-r)M_2}{M^2(M-1)(M+2)},
    \end{split} \label{var(alpha nu) j<=n1; Done}
\end{align}
with \( C \leq 2 + 6(1 + \sqrt{2})\sqrt{\frac{M-r}{rM_2M}} \).

\subsubsection{Cases for $\sigma_\eta^2$}
Assuming that $\bm{\alpha} \triangleq \bm{\alpha}(M) = \mathbf{U}^{(2)}\mathbf{u}_M$ is independent of $\bm{\tau}$, and $\bm{\tau}^{(2)} = (\mathbf{I}_{M_2}-\bm{\pi})\mathbf{U}^{(2)}$. In this case, we have 
\begin{align}
\begin{split}
\sigma_\eta^2 & = \frac{1}{r} \bm{\kappa}^\top\bm{\tau}\bm{\tau}^\top\bm{\kappa}\\
& = 
\frac{1}{r}(\bm{\alpha}^\top\bm{\tau}^{(2)}{\bm{\tau}^{(2)}}^\top\bm{\alpha} - 2\bm{\tau}_M^\top {\bm{\tau}^{(2)}}^\top \bm{\alpha} + \bm{\tau}_M^\top\bm{\tau}_M )\\
& \triangleq 
\frac{1}{r}(\mathrm{term}_1 -2\mathrm{term}_2 + \mathrm{term}_3)
\end{split}
\end{align}
It is straightforward to varify that:
\begin{align*}
\mathbb{E}(\mathrm{term}_3) &= \mathbb{E}(H_{MM}+H_{\varphi(M)\varphi(M)}-2H_{M\varphi(M)})\\
& = \frac{2r}{M} 
\\
\mathbb{E}(\mathrm{term}_2) &= \sum_{j=M_1+1}^M \mathbb{E}\left[
\left(H_{Mj}+H_{\varphi(M)\varphi(j)} -H_{\varphi(M)j}-H_{M\varphi(j)}\right)H_{jM}
\right] \\
& = \frac{r(M-r)(M_2-1)}{M(M-1)(M+2)}+\frac{2r(M-r)}{M^2(M+2)}-\frac{2r(M-r)}{M^2(M-1)(M+2)}\\
& = \frac{[M_2M+M-4]r(M-r)}{M^2(M-1)(M+2)}.
\end{align*}

Here we discuss a lower bound of $\mathbb{E}(\mathrm{term_1})$


\begin{theorem}\label{theorem: lower bound of term1}
$\mathbb{E}(\mathrm{term_1})
\geq 
\frac{r}{M_2}
\mathbb{E}
\left[({\mathbf{U}^{(2)}}^\top\mathbf{U}^{(2)})^{\frac{3}{r}}\right]$, where
\begin{align}
\begin{split}
\mathbb{E}
\left[({\mathbf{U}^{(2)}}^\top\mathbf{U}^{(2)})^{\frac{3}{r}}\right] 
&= \frac{\Gamma_r\left(\frac{r}{2} + \frac{3}{r}\right)\Gamma_r\left(\frac{M}{2}\right)}{\Gamma_r\left(\frac{M}{2} + \frac{3}{r}\right)\Gamma_r\left(\frac{r}{2}\right)} \\
\end{split}
\end{align}
\end{theorem}
\begin{proof}
\begin{align*}
\mathrm{term_1}
&= \bm{\alpha}^\top\bm{\tau}^{(2)}{\bm{\tau}^{(2)}}^\top\bm{\alpha} \\
&= \sum_{i=M_1+1}^{M}\sum_{j=M_1+1}^{M}
\left[
H_{Mi}H_{ij}H_{jM}+H_{Mi}H_{\varphi(i)\varphi(j)}H_{jM} -H_{Mi}H_{i\varphi(j)}H_{jM}-H_{Mi}H_{\varphi(i)j}H_{jM}
\right]
\end{align*}
which, by Corollary \ref{corollary:third-order moments}, leads to 
\begin{align*}
\mathbb{E}(\mathrm{term_1})
&= \bm{\alpha}^\top\bm{\tau}^{(2)}{\bm{\tau}^{(2)}}^\top\bm{\alpha} \\
&= \mathbb{E}\left[
\sum_{i=M_1+1}^{M}\sum_{j=M_1+1}^{M} H_{Mi}H_{ij}H_{jM} + \sum_{i=M_1+1}^MH_{Mi}^2H_{ii}
\right]\\
&\left(\approx\left(\frac{M_2}{M}\right)^2\tilde{\mathrm{c}}_3(r,M-r,M)+\mathbb{E}\left( \sum_{i=M_1+1}^MH_{Mi}^2H_{ii}\right)\right)
\\
&\geq \mathbb{E}\left[
\sum_{i=M_1+1}^{M}\sum_{j=M_1+1}^{M} \mathbf{u}_M^\top\mathbf{u}_i\mathbf{u}_i^\top\mathbf{u}_j\mathbf{u}_j^\top\mathbf{u}_M
\right]\\
&= \mathbb{E}\left[
    \mathrm{Tr}\left(
    ({\mathbf{U}^{(2)}}^\top\mathbf{U}^{(2)})^2\mathbf{u}_M\mathbf{u}_M^\top
    \right)
\right]\\
&= \frac{1}{M_2}
\mathbb{E}\left[
    \mathrm{Tr}\left(
    ({\mathbf{U}^{(2)}}^\top\mathbf{U}^{(2)})^3
    \right)
\right].
\end{align*}
Clearly ${\mathbf{U}^{(2)}}^\top\mathbf{U}^{(2)} \sim \mathcal{B}_r(\frac{r}{2}, \frac{M-r}{2})$. Denote its eigenvalues as $\lambda_1, \ldots, \lambda_r \geq 0$, we have
\begin{align*}
\mathrm{Tr}\left[
\left({\mathbf{U}^{(2)}}^\top\mathbf{U}^{(2)}\right)^3
\right]
&=\sum_{k=1}^r \lambda_k^3 \\
&\geq r\cdot \left(\lambda_1\lambda_2...\lambda_r\right)^\frac{3}{r} \\
&= r\cdot \left[
\mathrm{det}\left({\mathbf{U}^{(2)}}^\top\mathbf{U}^{(2)}\right)
\right]^\frac{3}{r}
\end{align*}
Thus, by leveraging Theorem \ref{theorem: E(det(H)^t)}, we have the lowerbound established.
\end{proof}

Hence, we can provide a lowerbound estimation of $\mathbb{E}(\sigma_\eta^2)$ as follows
\begin{align}
\mathbb{E}(\sigma_\eta^2)
\geq
\frac{2}{M}
-
2\frac{[M_2M+M-4](M-r)}{M^2(M-1)(M+2)}
+
\frac{
\Gamma_r\left(\frac{r}{2}+\frac{3}{r}\right)\Gamma_r\left(\frac{M}{2}\right)
}{
M_2\Gamma_r\left(\frac{M}{2} + \frac{3}{r}\right)\Gamma_r\left(\frac{r}{2}\right)
}
\end{align}

However, the multivariate gamma function $\Gamma_r(\cdot)$ may lead to computational issues due to its tendency to grow excessively large. To address this, we approximate a lower bound by assuming independence between $\bm{\tau}$ and $\bm{\kappa}$.

Specifically, we have:
\begin{align}
\begin{split}
\mathbb{E}(\sigma_\eta^2)
&=
\frac{1}{r}\mathbb{E}(\bm{\kappa}^\top \bm{\tau} \bm{\tau}^\top \bm{\kappa})\\
&\approx
\frac{1}{r}\mathbb{E}(\bm{\kappa})^\top \cdot \mathbb{E}_\tau(\bm{\tau}\bm{\tau}^\top) \cdot \mathbb{E}(\bm{\kappa}) + \frac{1}{r}\mathbb{E}_\tau\left[\mathrm{Tr}\left(\bm{\tau}\bm{\tau}^\top \cdot \mathrm{Cov}(\bm{\kappa})\right)\right]\\
&= \frac{(M-r)(M_2-1)}{M(M-1)(M+2)}\frac{\mathbb{E}_\tau(\|\bm{\tau}_{M-1}\|_2^2)}{r} + \left(\left(\frac{r}{M}-1\right)^2+\frac{2r(M-r)}{M^2(M+2)}\right)\frac{\mathbb{E}_\tau(\|\bm{\tau}_M\|_2^2)}{r} \\
&=\frac{2}{M}\left[
\frac{r(M-r)(M_2-1)}{M(M-1)(M+2)} + \left(\frac{r}{M}-1\right)^2 +\frac{2r(M-r)}{M^2(M+2)}
\right]
\end{split}
\end{align}

To mitigate the error introduced by ignoring the dependency between $\bm{\kappa}$ and $\bm{\tau}$, we empirically introduce a factor $\frac{M-M_2}{2M}$, yielding the approximation:
\begin{align}
\begin{split}
\mathbb{E}(\sigma_\eta^2)
& \approx
\frac{M-M_2}{M^2}\left[
\frac{r(M-r)(M_2-1)}{M(M-1)(M+2)} + \left(\frac{r}{M}-1\right)^2 +\frac{2r(M-r)}{M^2(M+2)}
\right]
\end{split}
\end{align}

\subsection{Theory Summary}

To sum up, we have
\begin{align}
\begin{cases}
\Pr\left(
    \sigma_\xi^2 < C\left(\frac{r(M-r)M_2}{M^2(M-1)(M+2)}\right)
    \right) 
> 
1-\delta
\\
\mathbb{E}(\sigma_\xi^2)
\leq
\frac{M_2r(M - r)}{M^2(M - 1)(M + 2)} + \frac{\sqrt{6}M_2r^{1/2}(M - r)^{3/2}}{M^2(M - 1)(M + 2)^{3/2}}+\frac{M_2^2}{M^3}
\\
\mathbb{E}(\sigma_\eta^2)
\geq
\frac{2}{M}
-
\frac{[2M_2M+M-4](M-r)}{M^2(M-1)(M+2)}
+
\frac{
\Gamma_r\left(\frac{r}{2}+\frac{3}{r}\right)\Gamma_r\left(\frac{M}{2}\right)
}{
M_2\Gamma_r\left(\frac{M}{2} + \frac{3}{r}\right)\Gamma_r\left(\frac{r}{2}\right)
}
\\
\mathbb{E}(\sigma_\eta^2)
 \approx
\frac{M-M_2}{M^2}\left[
\frac{r(M-r)(M_2-1)}{M(M-1)(M+2)}+\left(\frac{r}{M}-1\right)^2 +\frac{2r(M-r)}{M^2(M+2)}
\right]
\end{cases}
\label{eq:final variance of target expression}
\end{align}
where $C \leq 2+6(1+\sqrt{2})\sqrt{\frac{M-r}{rM_2M}}$, and 
\begin{align*}
    \Gamma_p(x) = \pi^{\frac{p(p-1)}{4}} \prod_{i=1}^p \Gamma\left(x - \frac{i-1}{2}\right).
\end{align*}
This matches the synthetic result shown in figure \ref{fig:orth_v9}. We can also see the robust effect from the modeled probability (\ref{Final Probability}) in Figure \ref{fig:orth_v15}

\begin{figure*}[t!]
    \centering
    \includegraphics[width=1.05\textwidth,trim={0 10 0 0},clip]{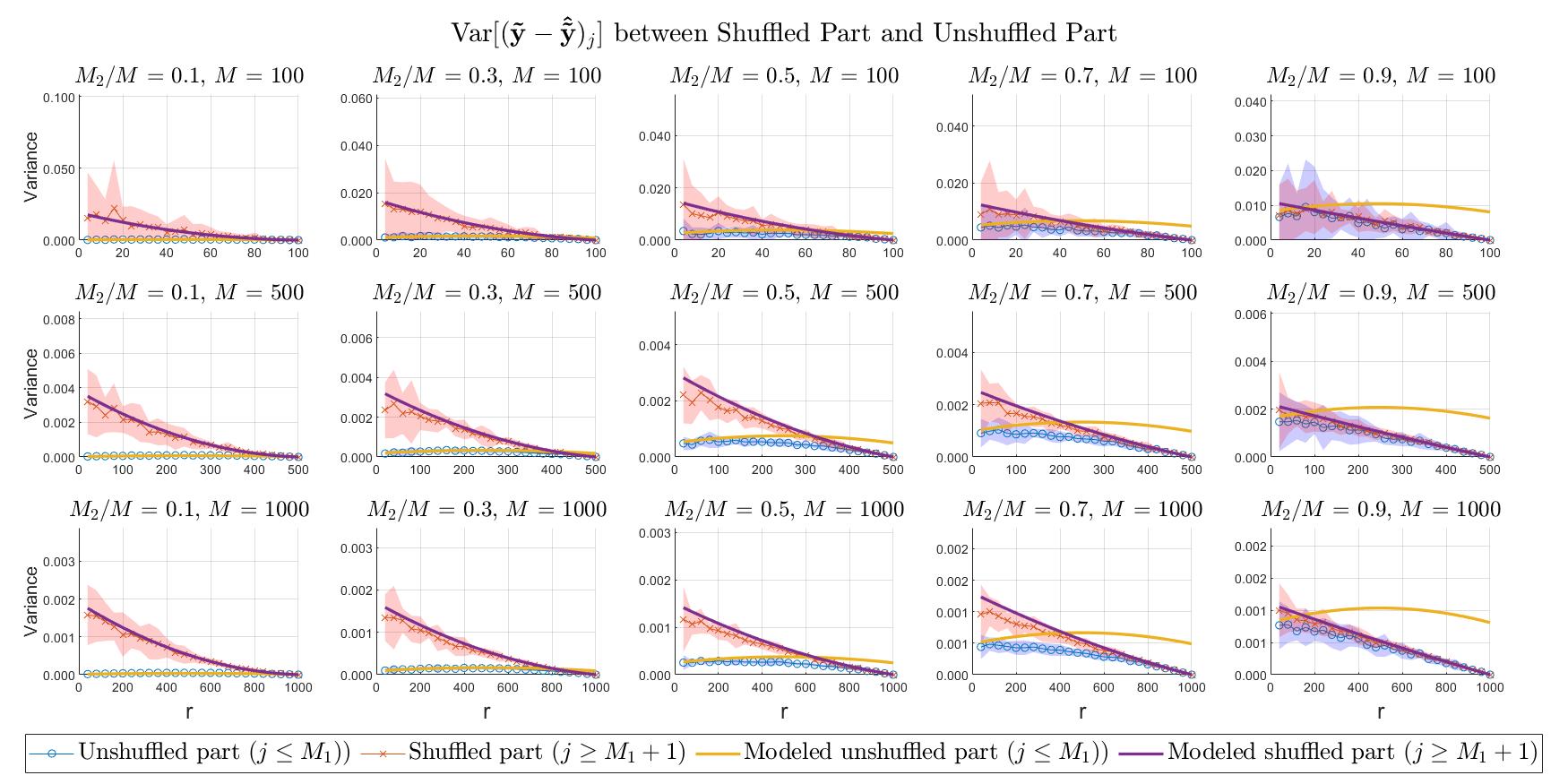}
    \caption{Comparison of $\mathrm{Var}\left[\left(\mathbf{\tilde{y}-\hat{\tilde{y}}}\right)_j\right]$ between Shuffled and Unshuffled Parts for Various $M$ and $M_2/M$ Ratios. Each subplot represents the variance of the difference $\left(\mathbf{\tilde{y}-\hat{\tilde{y}}}\right)_j$ for both the shuffled (orange crosses) and unshuffled (blue circles) parts across different values of $r$. The results consistently show that the variance is higher in the shuffled part compared to the unshuffled part, particularly as $r$ increases. The variance and its changing rate of the shuffled part remain significantly non-zero even when $r$ approaches one. This observation aligns with \eqref{eq:final variance of target expression}. Additionally, as the $M_2/M$ ratio increases, the variance of the unshuffled part gradually approaches that of the shuffled part, also consistent with the theoretical results. This highlights the effect of shuffling on the variance and underscores the importance of considering the shuffled indices group $\{ j \in \mathbb{Z}^{+}: M_1+1 \leq j \leq M \}$.}
\label{fig:orth_v9}
\end{figure*}

\begin{figure}[h!]
\centering
\includegraphics[width=0.8\textwidth,trim={0 40 0 0},clip]{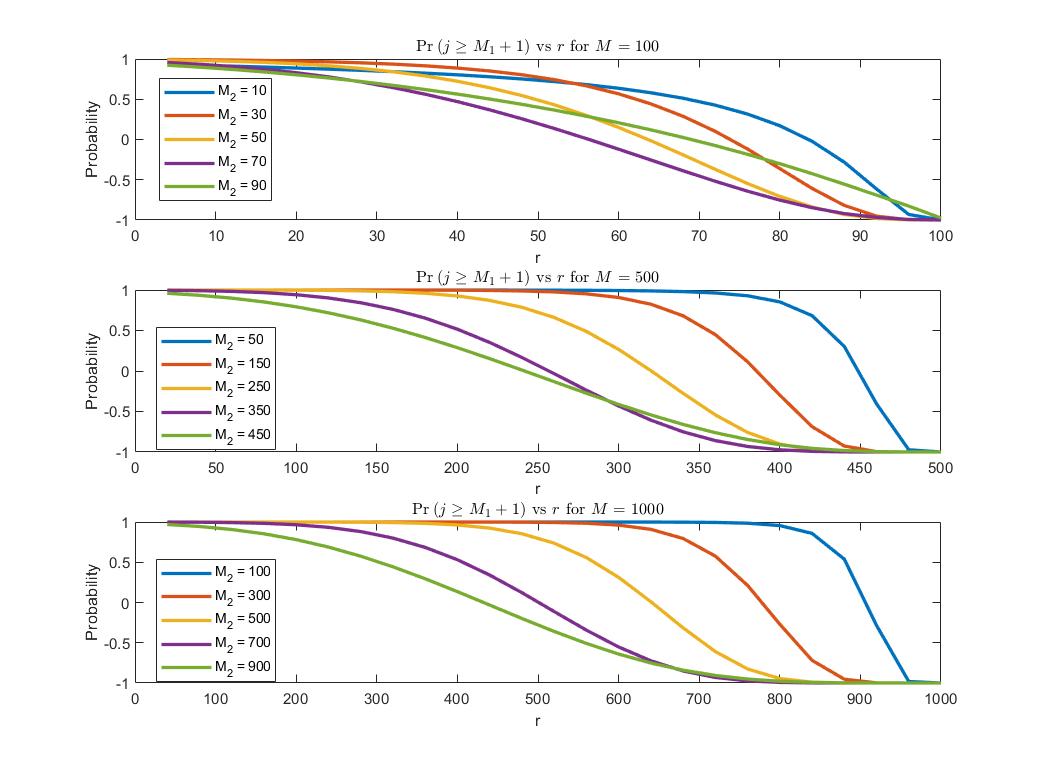}
\caption{
The modeled probability derived from (\ref{Final Probability}) demonstrates that the method maintains its effectiveness in low-rank and low-shuffled-ratio conditions. This observation aligns closely with our experimental findings, reinforcing the robustness of the proposed approach in scenarios where data complexity is constrained.
}
\label{fig:orth_v15}
\end{figure}

Throughout the above analyses, we have made several approximated assumptions regarding independence. These assumptions simplify the calculation of the variance of the target expression \eqref{eq:target_expression}, but they do not affect the accuracy of the results because our focus is on the scale, as shown in \eqref{eq:final variance of target expression}, rather than the exact values. Specifically, as the Cauchy-Schwarz Inequality states,
\[
\mathrm{Cov}\left(X, Y\right) \leq \sqrt{\mathrm{Var}\left(X\right)\mathrm{Var}\left(Y\right)},
\]
we understand that ignoring dependencies does not alter the magnitude order of the variance.

In summary, we have developed an approximate theoretical framework for analyzing the outlier classification problem using least squares regression. By introducing the permutation matrix and leveraging the properties of the matrix variate beta distribution, we derived expressions for the residuals of each entry and examined their statistical properties under various scenarios. Our findings suggest that the variance in the shuffled part is typically higher than in the unshuffled part, especially in scenarios where the dataset is partially shuffled with a low shuffled ratio.


Consequently, we have shown that, with high probability, the maximum absolute residual is likely to occur in the shuffled part in low rank low shuffled ratio scenarios. This finding suggests that the elimination process will make the remaining target vector more aligned with the column space of the remaining basis matrix. However, the stopping criteria in Algorithm \ref{alg:solve_partial_lsr}, defined by the retaining ratio $\gamma$, are robust but critical: stopping too late may cause the target vector to approach multiple candidate subspaces, while stopping too early may result in insufficient distinguishability, thereby hindering correct classification.

Moreover, our synthetic experiments validate the effectiveness of our approximation methods, demonstrating the robustness of using statistical distributions for approximation purposes across various parameter settings. Overall, our theoretical analysis and experimental results provide a solid foundation for both our outlier classification algorithm and the LSFR method proposed by \cite{upca}.


\end{document}